%% file: WACV_2027.tex
\documentclass[10pt,twocolumn,letterpaper]{article}

\usepackage[pagenumbers]{wacv} 

\input{preamble}

\newtheorem{theorem}{Theorem}[section]
\newtheorem{lemma}[theorem]{Lemma}

\definecolor{wacvblue}{rgb}{0.21,0.49,0.74}
\usepackage[pagebackref,breaklinks,colorlinks,allcolors=wacvblue]{hyperref}

\def\confName{WACV}
\def\confYear{2027}

\title{T-QPM: Enabling Temporal Out-Of-Distribution Detection and Domain Generalization for Vision-Language Models in Open-World}

\author{Aditi Naiknaware\\
San Diego State University\\
{\tt\small anaiknaware7153@sdsu.edu}
\and
Salimeh Sekeh\\
San Diego State University\\
{\tt\small ssekeh@sdsu.edu}
}

\begin{document}
\maketitle
\input{sec/0_abstract} 
\input{sec/1_intro}
\input{sec/2_related_work} 

\input{sec/3_methodology}

\input{sec/4_theory}

\input{sec/5_experiments}

\input{sec/6_conclusion}
\input{sec/Acknowledgement}
{
    \small
    \bibliographystyle{ieeenat_fullname}
    \bibliography{main}
}
\clearpage
\input{sec/X_supp}

\end{document}

%% file: preamble.tex
\usepackage{wrapfig}

\usepackage{graphicx}
\usepackage{makecell}
\usepackage{amsthm}

\usepackage[accsupp]{axessibility}

\usepackage{multirow}
\usepackage{algorithm}
\usepackage{algpseudocode}
\usepackage{tcolorbox}
\usepackage{placeins}
\tcbuselibrary{skins} 
\usepackage{placeins}
\usepackage{float}

\usepackage{booktabs}

%% file: sec/0_abstract.tex
\begin{abstract}

Out-of-distribution (OOD) detection remains a critical challenge in open-world learning, where models must adapt to evolving data distributions. While recent vision-language models (VLMS) like CLIP enable multimodal OOD detection through Dual-Pattern Matching (DPM), existing methods typically suffer from two major shortcomings: (1) They rely on fixed fusion rules and assume static environments, failing under temporal drift; and (2) they lack robustness against covariate shifted inputs. In this paper, we propose a novel two-step framework to enhance OOD detection and covariate distribution shift robustness in dynamic settings. We extend the dual-pattern regime into Temporal Quadruple-Pattern Matching (T-QPM). First, by pairing OOD images with text descriptions, we introduce cross-modal consistency patterns between ID and OOD signals, refining the decision boundary through joint image-text reasoning. Second, we address temporal distribution shifts by learning lightweight fusion weights to optimally combine semantic matching and visual typicality. To ensure stability, we enforce explicit regularization based on Average Thresholded Confidence (ATC), preventing performance degradation as distributions evolve. Experiments on temporally partitioned benchmarks demonstrate that our approach significantly outperforms static baselines, offering a robust, temporally-consistent framework for multimodal OOD detection in non-stationary environments.
\end{abstract}

%% file: sec/1_intro.tex
\section{Introduction}
\label{sec:intro}

As the demand for intelligent systems grows, the need for computer vision algorithms and foundation models to handle open-world scenarios becomes increasingly paramount. One important characteristic of the open world for vision-language models (VLMs) is that intelligent systems will encounter new contexts and images that were not seen during training, requiring safe handling of unseen examples (out-of-distribution (OOD)  detection) and adaptation to distribution-shifted inputs (domain generalization) in temporal environments~\cite{zhu2024owl}. Noticeably, the vast majority of VLMs have been driven by the closed-world setting~\cite{radford2021clip, ming2022mcm}, where the label space is assumed fixed and the data distribution stationary. An open-world learning (OWL) paradigm on wild data~\cite{katz2022training} is built upon two parts: unknown rejection (OOD detection) and novel class discovery (distribution shift generalization) under dynamic domains. Within the OWL context, in-distribution (ID) refers to data drawn from the same distribution as the training set---the data that the model is expected to handle reliably. Prior work in both OOD detection and distribution shift has primarily focused on two categories: (1) \textbf{covariate shift} refers to inputs that belong to the same label space as the training data but differ due to changes in the input distribution~\cite{ye2022ood, koh2021wilds}, such as a dog image corrupted with Gaussian noise remaining labeled as ``dog'' yet degrading model performance; and (2) \textbf{semantic shift} occurs when entirely new classes are introduced at test time~\cite{yang2024generalized, ye2022ood}, such as a classifier trained on cats and dogs encountering an elephant.

While recent advances in OOD detection for VLMs have shown great promise.
Notably, Maximum Concept Matching (MCM)~\cite{ming2022mcm} leverages softmax-scaled cosine similarity between visual and textual concept prototypes for zero-shot OOD detection, and Dual-Pattern Matching (DPM)~\cite{zhang2023dpm} efficiently adapts CLIP for OOD detection by exploiting both visual and textual ID patterns, these methods lack several fundamental aspects of OWL: (1) they largely overlook temporal dynamics, the fact that data distributions may evolve over time due to changing environments, user behavior, or data sources~\cite{yao2022wild}; (2) they neglect covariate shift and domain generalization during OOD detection;
and (3) they limit OOD evaluation to unimodal image inputs, leaving the rich linguistic signal of VLMs underexploited. Without mitigation, temporal shifts can cause gradual but systematic performance degradation.
For example, a perception system trained on one year's traffic patterns may underperform as road construction, seasonal changes, or evolving driving behaviors shift the data distribution over time~\cite{yao2022wild, cai2024continuous}.

{\bf Our Contribution:} In this paper, we propose \textbf{T-QPM} (Temporal Quadruple Pattern Matching), a novel multimodal OOD detection framework designed for open-world deployment under continuously evolving distributions. T-QPM builds on frozen CLIP backbones and operates over image-caption pairs, enabling richer cross-modal interaction than DPM approach. T-QPM explicitly models temporal distribution shift by incorporating a caption-aware temporal regularization loss that stabilizes confidence-based decision boundaries across timesteps, jointly optimizing for ID classification, covariate shift robustness, temporal consistency, and semantic OOD detection. We provide theoretical study linking temporal consistency to generalization error bound. Our experimental results show T-QPM outperforms DPM baseline significantly on wild data.

%% file: sec/2_related_work.tex
\section{Related Work}

Recent CLIP-based OOD detectors exploit vision-language representations through zero-shot scoring~\cite{ming2022mcm,esmaeilpour2022zero}, global-local feature matching~\cite{miyai2025glmcm}, negation semantics~\cite{wang2023clipn}, prompt learning~\cite{miyai2023locoopfewshotoutofdistributiondetection,jiang2024negativelabelguidedood,li2024negprompt}, and self-calibrated tuning to suppress spurious features~\cite{yu2024sct}. 
DPM~\cite{zhang2023dpm} most directly informs T-QPM by combining visual and textual ID patterns, which we extend to four cross-modal signals with temporal awareness. 
For open-world generalization under distribution shift, SCONE~\cite{bai2025feedbirdssconeexploiting} jointly handles covariate and semantic shift on wild data, Temp-SCONE~\cite{naiknaware2025tempscone} extends this to temporal settings, HYPO~\cite{bai2024hypo} learns provably domain-invariant hyperspherical representations, and Meta-OOD~\cite{wu2023meta} enables few-shot adaptation to novel OOD distributions. 
On the temporal side, benchmarks such as CLEAR~\cite{lin2022clearbenchmarkcontinuallearning}, Wild-Time~\cite{yao2022wild}, and WILDS~\cite{koh2021wilds} motivate methods like~\cite{cai2024continuous} and CODA~\cite{chang2025coda}, which tackle continuously shifting and concept-drifting domains without target-domain access, but none address multimodal OOD detection, which T-QPM uniquely unifies with temporal robustness in a single caption-aware framework. Extended related work is provided in the supplementary material (SM). 

%% file: sec/3_methodology.tex
\section{Methodology}
\subsection{Preliminaries and Problem Setup}
We start with preliminaries to lay the necessary context, followed by a clear description of VLMS for OOD detection. 
We consider a deployed classifier $f_\theta: \mathcal{X} \rightarrow \mathbb{R}^K$ trained on a labeled in-distribution (ID) dataset $ \mathcal{D}_{\text{ID}} = \{(x_i, y_i)\}_{i=1}^n$, drawn \textit{i.i.d.} from the joint data distribution $\mathbb{P}_{\mathcal{XY}}$. The function $f_\theta$ predicts the label of an input sample $\mathbf{x}$ as $\hat{y}(f(\mathbf{x})):=\arg\max_y f^y(\mathbf{x})$. Define $\mathbb{P}_{\hbox{in}}$, the marginal distribution of the labeled data $(\mathcal{X}, \mathcal{Y})$, which is also referred to as the in-distribution. $\mathbb{P}_{\hbox{out}}^{type}$ is the marginal distribution out of $\mathbb{P}_{\mathcal{X}'\mathcal{Y}'}$ on $\mathcal{X}'$, where the input space undergoes "type" shifting and the joint distribution has the same label space or different label space (depending to the "type"). 
We consider a generalized characterization of the open world setting with two types of OOD 
\begin{equation}\label{eq.wild}
{\mathbf{P}}_{\hbox{wild}}= (1-\sum_{type} \pi_{type}) {\mathbb{P}}_{\hbox{in}} +\sum_{type}\pi_{type} {\mathbb{P}}_{\hbox{out}}^{type}, 
\end{equation}
where $type=\{\hbox{semantic},\hbox{covariate}\}$, where $\pi_{type}, \sum\limits_{type}\pi_{type}\in(0,1)$.\\
{\it Covariate OOD type:} Taking autonomous driving as an example, a model trained on ID data with sunny weather may experience a covariate shift due to foggy/snowy weather. Under such a covariate shift, a model
is expected to generalize to the OOD
data—correctly predicting the sample into one of the known
classes (e.g., car), despite the shift. $\mathbb{P}_{\hbox{out}}^{cov}$ is the marginal distribution
of covariate shifted data $(\mathcal{X}',\mathcal{Y})$ with distribution 
$\mathbb{P}_{\mathcal{X}' \mathcal{Y}}$, where the joint distribution has the same label space as the training data, yet the input
space undergoes shifting in domain. \\
{\it Semantic OOD type:} In autonomous driving example,  the model
may encounter a semantic shift, where samples are from
unknown classes (e.g., bear) that the model has not been
exposed to during training. $\mathbb{P}_{\hbox{out}}^{sem}$ is the marginal distribution when wild data does not belong to any known categories $Y= \{1,2,...,K\}$ and therefore should be detected as OOD sample. To detect the semantic OOD data, we train OOD detector which is a ranking function $g_\theta: \mathcal{X}\mapsto \mathbb{R}$ with parameter $\theta$: if  $g_\theta(\mathbf{x}) \leq \lambda$ then the sample $\mathbf{x}$ is OOD example. 
The threshold value $\lambda$ is typically chosen so that a high fraction of ID data is correctly classified. This means that the detector $g_\theta$ should predict semantic OOD data as OOD and o.w predict as ID. 

\noindent\textbf{VLMs for OOD Detection.} VLMs exemplified by CLIP, consist of two aligned encoders: a visual encoder $\phi^V$ and a text encoder $\phi^T$. The visual encoder maps an input image $x$ to a $d$-dimensional feature representation $\mathbf{F}_v(x) \in \mathbb{R}^d$, while the text encoder maps a textual prompt $p$ to a semantic embedding $\phi^T(p) \in \mathbb{R}^d$. Both embeddings lie in a shared representation space, enabling direct similarity comparison. For a downstream classification task with label space $\mathcal{Y} = \{y_1, \ldots, y_K\}$, a prompt template is instantiated for each class to obtain class-specific textual descriptions. Each class embedding is normalized to produce a set of prototype vectors $\{\mathbf{t}_k\}_{k=1}^K$, where $\mathbf{t}_k \in \mathbb{R}^d$. These normalized text embeddings form a cosine-similarity classifier in the joint embedding space.

Given an image $x$, the compatibility between the visual feature $\mathbf{F}_v(x)$ and each class prototype $\mathbf{t}_k$ is measured via cosine similarity. These similarity scores are scaled by a temperature parameter and converted into a categorical distribution over $\mathcal{Y}$ using a softmax transformation. The resulting predictive distribution reflects the semantic alignment between the image and the set of textual class descriptions.
In CLIP-based OOD detection, the text embeddings $\{\mathbf{t}_k\}$ serve as fixed classifier weights. Post-hoc scoring functions such as Maximum Softmax Probability (MSP)~\cite{hendrycks2017baseline} or Energy-based scores~\cite{liu2020energy} are applied to the resulting logits or probability distribution. The underlying assumption is that ID samples yield higher confidence or lower energy compared to OOD samples.

\noindent\textbf{Problem Setup.}
We consider the problem of OOD detection under \textit{temporal
distribution shift}. At each timestep $t \in \{0, 1, \ldots, \mathcal{T}\}$,
a data distribution $\mathcal{D}_t$ over image-caption pairs $(x, c)$
is observed, where the visual distribution shifts gradually across
timesteps. The ID label set
$\mathcal{Y} = \{y_1, \ldots, y_K\}$ remains fixed, but the visual
appearance of ID classes evolves over time, inducing \textit{temporal shift}. At test time, each pair $(x, c)$ must be classified
as ID or OOD with respect to $\mathcal{Y}$, without access to OOD
labels during training.

Formally, at each timestep $t$ we have a set of labeled ID
training and covatiate shifted samples $\mathcal{D}_t^{\mathrm{train}} = \{(x_i, y_i)\}$
drawn from the current ID distribution, along with unlabeled test
samples $\{(x, c)\}$ that may be either ID or OOD. The goal is to
learn a scoring function $S(x, c, t) \in \mathbb{R}$ such that ID
samples consistently score above a fixed threshold $\delta$ across
all timesteps, while OOD samples score below it, and such that the
scoring function remains robust to both temporal drift and covariate
perturbations of the input.

Our {\bf T-QPM} method 
extends DPM~\cite{zhang2023dpm} to this setting by: (i) incorporating
all four cross-modal pairings between ID and test representations
including image and caption modalities, (ii) adapting visual prototypes
per timestep to handle temporal drift, (iii) learning a lightweight
fusion of the four scores with only 2 trainable parameters, and (iv)
enforcing covariate robustness through explicit consistency
regularization during training.

\begin{figure*}[!t]
   \centering
    \includegraphics[width=\linewidth]{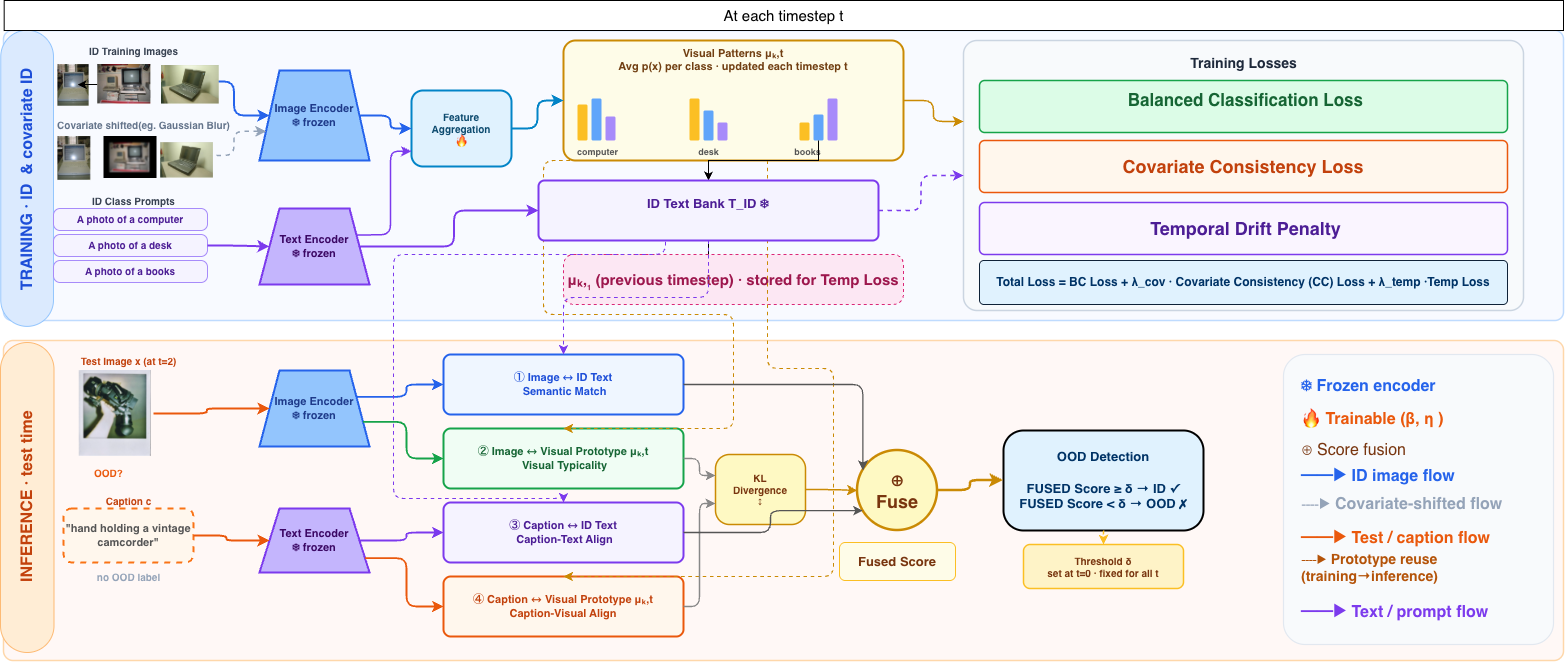}
    \caption{{\textbf{T-QPM Overview:} At each timestep, ID images and their covariate-shifted views are encoded to build timestep-specific visual prototypes alongside a fixed ID Text Bank. At inference, four cross-modal scores between the test image, caption, and ID representations are fused to produce the final OOD decision.}}
    \label{fig:overview} 
\end{figure*}

\subsection{T-QPM for OOD Detection}
Our T-QPM method is designed mainly based on four phases as described below. \\
\noindent\textbf{Phase I: Text Pattern Construction.}
We first construct reference patterns using the text encoder of the frozen VLM. For each ID class $k \in \{1,\dots,K\}$,
we employ prompt ensembling over $P$ templates to obtain robust text representations. Let $\{p_k^{(1)}, \dots, p_k^{(P)}\}$ be the set of prompts for class $k$. The class text embedding is computed as:
\begin{align}\label{eq:class-text-embedding}
\mathbf{t}_k = \mathrm{Normalize}\!\left(\sum_{i=1}^{P}
\mathrm{Normalize}\!\left(\phi^T(p_k^{(i)})\right)\right),
\end{align}
where $\phi^T$ denotes the frozen CLIP text encoder and
normalization is the $\ell_2$-norm. Collecting all ID class embeddings yields
the \textit{ID text bank}
$\mathbf{T}^{\mathrm{ID}} = [\mathbf{t}_1, \dots, \mathbf{t}_K]^\top \in \mathbb{R}^{K \times d}$,
where $d$ is the embedding dimension. $\mathbf{T}^{\mathrm{ID}}$ is computed
once and kept fixed throughout all timesteps, serving as a stable semantic
anchor while temporal variability is handled through timestep-specific
visual statistics. 

\noindent\textbf{Phase II: Temporal Visual Pattern Construction.}
To account for temporal distribution shift~\cite{wang2022learningpromptcontinuallearning}, we
compute visual reference statistics separately for each timestep, extending
the DPM framework. Given an
image $x$ at timestep $t$, the frozen CLIP-ViT encoder
produces a sequence of patch embeddings
$\mathbf{F}(x) = \phi^V(x) \in \mathbb{R}^{(N+1) \times d}$,
where $N$ is the number of spatial patches and the first token is the global
\texttt{[CLS]} token. We decompose $\mathbf{F}(x)$ into
a \emph{global token} $\mathbf{F}_v(x) = \mathbf{F}(x)[0,:] \in \mathbb{R}^{d}$
and \emph{spatial patches} $\mathbf{F}_s(x) = \mathbf{F}(x)[1:,:] \in \mathbb{R}^{N \times d}$.
For each class $k$, class-specific spatial attention weights are computed as
\begin{equation}
\mathbf{A}_k(x) = \mathrm{Softmax}\!\left(
    \frac{\mathbf{F}_s(x)\,\mathbf{t}_k}{\|\mathbf{F}_s(x)\|\,\|\mathbf{t}_k\|}
\right) \in \mathbb{R}^{N},
\end{equation}
yielding a class-attended spatial feature
$\tilde{\mathbf{f}}_k(x) = \mathbf{A}_k(x)^\top \mathbf{F}_s(x) \in \mathbb{R}^{d}$.
The final class-specific image representation combines global and attended
spatial features as $\mathbf{f}_k(x) = \gamma\,\tilde{\mathbf{f}}_k(x) + \mathbf{F}_v(x)$,
where $\gamma$ balances the contribution of spatial and global information,
and the ID logits follow as
\begin{equation}
\mathbf{z}_{\mathrm{ID}}(x) =
\bigl[\mathbf{f}_1(x)^\top \mathbf{t}_1,\; \ldots,\;
      \mathbf{f}_K(x)^\top \mathbf{t}_K\bigr]^\top \in \mathbb{R}^{K}.
\end{equation}
Converting logits to probabilities via temperature-scaled
softmax,
$\mathbf{p}(x) = \mathrm{Softmax}\!\left(\mathbf{z}_{\mathrm{ID}}(x) / T\right)$,
we estimate for each timestep $t$ and class $k$ a class-conditional visual
prototype by averaging over ID training samples from $\mathcal{D}^{\mathrm{train}}_t$:
$\{\boldsymbol{\mu}_{k,t}\}_{k=1}^{K}$, $\boldsymbol{\mu}_{k,t} =
\mathbb{E}\bigl[\mathbf{p}(x) \mid (x,y) \in \mathcal{D}_t^{\mathrm{train}},\;
y=k\bigr] \in \mathbb{R}^{K}$ defines the expected ID similarity pattern at time $t$, allowing the notion of ``normal'' ID
behavior to adapt to gradual visual distribution drift.

\noindent\textbf{Phase III: Quadruple Cross-Modal Scoring.}
For a test image $x$ with associated caption $c$ observed at timestep $t$,
we construct four complementary OOD scores corresponding to all cross-modal
pairings between ID and test representations, together realizing the complete
Quadruple Pattern Matching (QPM) framework in the temporal
setting. The \textit{Semantic Matching Score} $S_{\mathrm{ID}}$ (\textbf{OOD
image $\leftrightarrow$ ID text}) measures alignment between the test image's
visual features and the ID class text embeddings,
\begin{equation}
S_{\mathrm{ID}}(x) = \max_{k \in \{1,\ldots,K\}}
\frac{\mathbf{z}_{\mathrm{ID}}(x)[k]}{T},
\end{equation}
where higher values indicate stronger semantic similarity to known ID classes.
The \textit{Visual Typicality Score} $S_{\mathrm{VIS}}$ (\textbf{OOD image
$\leftrightarrow$ ID visual}) measures how typical the test image's probability
pattern is relative to the timestep-specific ID visual prototypes via KL
divergence~\cite{sun2022knnood},
\begin{align}
\mathrm{KL}_k(x,t) = \sum_{j=1}^{K} p_j(x)\log\frac{p_j(x)}{\mu_{k,t}[j]},\nonumber\\
S_{\mathrm{VIS}}(x,t) = -\min_{k}\;\mathrm{KL}_k(x,t),
\end{align}
where $p_j(x)$ denotes the $j$-th entry of $\mathbf{p}(x)$, and a lower KL
divergence (higher $S_{\mathrm{VIS}}$) indicates the test image's probability
pattern is consistent with typical ID images at timestep $t$. To exploit the
multimodal nature of the test data, each test image is accompanied by a natural
language caption $c$, which we encode on-the-fly using the same frozen text
encoder as $\mathbf{q}_c = \mathrm{Normalize}(\phi^T(c)) \in \mathbb{R}^d$.
The \textit{Caption-Text Alignment Score} $S_{\mathrm{CAP\text{-}T}}$
(\textbf{OOD text $\leftrightarrow$ ID text}) then measures how strongly the
caption's semantics overlap with ID class names in text space,
\begin{equation}
S_{\mathrm{CAP\text{-}T}}(x) =
\max_{k \in \{1,\ldots,K\}}\;\langle \mathbf{q}_c,\, \mathbf{t}_k \rangle,
\end{equation}
with no OOD text bank constructed --- the caption is compared directly against
the precomputed $\mathbf{T}^{\mathrm{ID}}$ at inference time. Complementing
this, the \textit{Caption-Visual Alignment Score} $S_{\mathrm{CAP\text{-}V}}$
(\textbf{OOD text $\leftrightarrow$ ID visual}) measures whether the caption's
semantics are consistent with the visual statistics of ID data at the current
timestep. Reusing $\mathbf{q}_c$, we project it through the ID logit computation
to obtain a caption probability vector,
\begin{align}
\mathbf{z}_{\mathrm{CAP}}(x) =
\bigl[\mathbf{q}_c^\top \mathbf{t}_1,\; \ldots,\;
      \mathbf{q}_c^\top \mathbf{t}_K\bigr]^\top \in \mathbb{R}^{K},\nonumber\\
\mathbf{p}_{\mathrm{CAP}}(x) =
\mathrm{Softmax}\!\left(\frac{\mathbf{z}_{\mathrm{CAP}}(x)}{T}\right),
\end{align}
and measure its typicality against $\boldsymbol{\mu}_{k,t}$ via KL divergence,
\begin{equation}
S_{\mathrm{CAP\text{-}V}}(x,t) =
-\min_{k}\;\mathrm{KL}\!\left(
\mathbf{p}_{\mathrm{CAP}}(x) \,\Big\|\, \boldsymbol{\mu}_{k,t}\right).
\end{equation}
Since both caption scores reuse $\mathbf{q}_c$ computed once per test image,
the multimodal grounding adds negligible overhead at inference.
$S_{\mathrm{CAP\text{-}T}}$ and $S_{\mathrm{CAP\text{-}V}}$ provide
complementary signals: the former detects semantic overlap in text space,
while the latter detects whether the caption's semantics conform to the
visual statistics of ID data at timestep $t$. The four scores are combined
with learnable positive weights $\beta, \eta > 0$ as
\begin{align}\label{eq:fused}
&S_{\mathrm{FUSED}}(x,t) = S_{\mathrm{ID}}(x) + \beta \cdot S_{\mathrm{VIS}}(x,t)\nonumber\\
&- \gamma_{\mathrm{cap}} \cdot S_{\mathrm{CAP\text{-}T}}(x)
- \eta \cdot S_{\mathrm{CAP\text{-}V}}(x,t),
\end{align}
where $\gamma_{\mathrm{cap}} > 0$ is a fixed hyperparameter and $\beta, \eta$
are learned. The caption-based terms are subtracted since high alignment of
the test caption with ID representations is indicative of an OOD sample
whose textual description overlaps with but whose visual content departs
from the ID distribution. To ensure positivity, $\beta$ and $\eta$ are
parameterized via the softplus function by $\beta = \log(1 + e^{\tilde{\beta}})$ and $\eta  = \log(1 + e^{\tilde{\eta}})$, 
where $\tilde{\beta}, \tilde{\eta} \in \mathbb{R}$ are trainable scalar
parameters initialized. 

\noindent\textbf{Phase IV: Threshold Calibration, Training Objective,
and Temporal OOD Detection.}
At the initial timestep $t=0$, we calibrate a decision threshold $\delta$
as the $\delta_q$-th percentile of fused scores on ID training data,
\begin{equation}
\delta = \mathrm{quantile}_{\delta_q}\!\bigl(
\{S_{\mathrm{FUSED}}(x,0) \mid (x,y) \in \mathcal{D}_0^{\mathrm{train}}\}
\bigr).
\end{equation}
 This threshold is
fixed across all subsequent times to enable consistent temporal comparison.
At each timestep $t$, we then optimize only the two fusion scalars
$\tilde{\beta}$ and $\tilde{\eta}$, keeping
all CLIP encoders frozen. This yields extreme parameter
efficiency and avoids catastrophic forgetting of pre-trained
representations~\cite{Kirkpatrick_2017}. Beyond OOD detection under temporal shift, the training
objective is explicitly designed to generalize over covariate shift by
training on both clean and corrupted views, enforcing their score consistency.
The total loss comprises three components. First, the
\textit{Balanced ID Classification Loss} computes cross-entropy symmetrically
on both clean and covariate-shifted views,
\begin{align}
\mathcal{L}_{\mathrm{ID}} =
\frac{1}{2}\Bigl(
  \mathbb{E}_{(x,y)\sim\mathcal{D}_t}
  \bigl[\mathcal{L}_{\mathrm{CE}}(\mathbf{z}_{\mathrm{ID}}(x),\, y)\bigr]\nonumber\\
+ \mathbb{E}_{(\tilde{x},y)\sim\mathcal{D}_t}
  \bigl[\mathcal{L}_{\mathrm{CE}}(\mathbf{z}_{\mathrm{ID}}(\tilde{x}),\, y)\bigr]
\Bigr),
\end{align}
where $\tilde{x}$ is the covariate-shifted of $x$
so that the model learns representations that are
simultaneously discriminative and robust to covariate perturbations. Second,
the \textit{Covariate Consistency Loss} explicitly enforces that OOD detection
scores remain stable under covariate shift~\cite{hendrycks2020augmixsimpledataprocessing},
\begin{equation}
\mathcal{L}_{\mathrm{COV}} =
\mathbb{E}_{x\sim\mathcal{D}_t}\Bigl[
  \bigl|S_{\mathrm{FUSED}}(x,t) - S_{\mathrm{FUSED}}(\tilde{x},t)\bigr|
\Bigr],
\end{equation}
directly penalizing inconsistency between the OOD scores of clean and
corrupted views and making the detection boundary robust to covariate
shifts. Third, to prevent the effective ID coverage from drifting
as the data distribution evolves, the \textit{Temporal Drift Penalty}
employs \textit{Above-Threshold Coverage (ATC)}~\cite{garg2022leveragingunlabeleddatapredict} 
as a soft differentiable proxy for the fraction of ID samples scoring
above $\delta$,
\begin{equation}
\mathrm{ATC}_t =
\mathbb{E}_{x\sim\mathcal{D}_t}\!\left[
  \sigma\!\left(\frac{\delta - S_{\mathrm{FUSED}}(x,t)}{\kappa}\right)
\right],
\end{equation}
where $\sigma(\cdot)$ is the sigmoid function and $\kappa > 0$ controls
the smoothness of the approximation. ATC is computed separately for clean
and covariate-shifted views, and changes across consecutive timesteps are
penalized as
\begin{equation}
\mathcal{L}_{\mathrm{TEMP}} =
\bigl|\mathrm{ATC}_t^{\mathrm{clean}} - \mathrm{ATC}_{t-1}^{\mathrm{clean}}\bigr|
+
\bigl|\mathrm{ATC}_t^{\mathrm{shift}} - \mathrm{ATC}_{t-1}^{\mathrm{shift}}\bigr|,
\end{equation}
discouraging abrupt changes to the fraction of ID samples above the detection
threshold at each timestep and stabilizing detection across the evolving data
stream while simultaneously accounting for covariate
robustness through the shifted ATC term. The total loss with
Lagrangian multipliers $\lambda_{\mathrm{cov}}$ and $\lambda_{\mathrm{temp}}$ is:
\begin{equation}
\mathcal{L}_{\mathrm{TOTAL}} =
\mathcal{L}_{\mathrm{ID}}
+ \lambda_{\mathrm{cov}}\,\mathcal{L}_{\mathrm{COV}}
+ \lambda_{\mathrm{temp}}\,\mathcal{L}_{\mathrm{TEMP}}.
\end{equation}
Given a test image $x$ with caption $c$ at timestep $t$, we compute all
four scores, form $S_{\mathrm{FUSED}}(x,t)$ via Eq.~\eqref{eq:fused},
and classify using the fixed threshold $\delta$:
\begin{equation}
D(x,t) =
\begin{cases}
\texttt{ID}  & \text{if } S_{\mathrm{FUSED}}(x,t) \geq \delta,\\
\texttt{OOD} & \text{otherwise.}
\end{cases}
\end{equation}
Together, the four cross-modal scores realize the complete QPM
framework in the temporal setting, with the training
objective jointly ensuring robustness to both temporal drift and covariate shift. 

\noindent The Pseudocode of all Phases I-IV are provided in SM.

%% file: sec/4_theory.tex
\section{Theory on Generalization Error}
Inspired by theoretical investigations in \cite{zhang2024best, tong2021mathematical}, we have studied generalization error ($GErr_{t+1}(f)$) of model $f_\theta$ for two time steps $t$ and $t+1$.
The generalization error at time step $t$, $GErr_t$, is standard cross entropy loss for hypothesis $f\in\mathcal{F}$ under covariant shift $\mathbb{P}^{cov}$.
We assume: {\bf [A1]} At time step $t$, $TV(p(y_t|x_t)\|\mathcal{U})$ is constant. {\bf [A2]} At time step $t$, $F_f^{\theta_1}$ The class distributions predicted by $f$ and $p^{\theta_2}(y_t|x_t)$ have same distribution with different parameter $\theta_1$ and $\theta_2$, respectively and $\theta_1-\theta_2=\delta$, where $\delta$ is bounded. {\bf [A3]} There exist a constant (say $Z_t$), s.t.
\begin{align*}
&\mathbb{E}_{\mathbb{P}_{out}^{t+1,cov}}H(p(y_{t+1}|x_{t+1}))-\mathbb{E}_{\mathbb{P}_{out}^{t,cov}}H(p(y_{t}|x_{t}))) \\
&\geq Z_t+ Conf_{t}- Conf_{t+1}.
\end{align*}
\begin{theorem}
\label{thm1}
{\bf (Main Theorem)} Let $\mathbb{P}^{t,cov}$ and $\mathbb{P}_{test}^{t,sem}$ be the covariate-shifted OOD and semantic OOD distribution.  Denote $GErr_{t+1}(f)$ the generalization error at time $t$. Let $\mathcal{L}_{reg}$ be the OOD detection loss devised for MSP detectors~\cite{hendrycks2019deepanomalydetectionoutlier}, i.e., cross-entropy
between predicted distribution $f_\theta$ and uniform distribution. Then at two time steps $t$ and $t+1$ and under assumptions {\bf [A1]}-{\bf [A3]},  we have
\begin{align}\label{thm1:main-eq}
 & GErr_{t+1}(f)-GErr_{t}(f) \geq - \tilde{\kappa}\; \Delta_{t\rightarrow t+1}^{cov,sem} - \tilde{\kappa}\; \Xi_{t\rightarrow t+1}^{sem} \nonumber\\
 & - \overline{\delta}_t^2 \;\mathbb{E}_{\mathbb{P}_{out}^{t,cov}}\left(I_F(\theta)\right)+C_{t\rightarrow t+1} + Conf_{t} - Conf_{t+1},  
  \end{align}
where $\Delta_{t\rightarrow t+1}^{cov,sem}$ is defined 
based on disparity discrepancy with total variation distance) (TVD) at timestep $t$ and $t+1$ that measures the dissimilarity of covariate-shifted OOD and semantic OOD. $\Xi_{t\rightarrow t+1}^{sem}$ is defined based on OOD detector.
  And $C_{t\rightarrow t+1}=C_{t+1}-C_t+B_t+Z_t$ and $\delta_t$ are constants and  $\overline{\delta}_t^2=\frac{log e}{2}\delta_t^2$. 
  \(Conf(f_\theta):=\max_{j\in \mathcal{Y}} f_j(x)\) is maximum confidence, and $I_f(\theta)$ is Fisher Info.~\cite{cramer1999mathematical}.
\end{theorem}
The details and proof are deferred in the SM. Our theoretical finding demonstrates that for MSP detectors (without any OOD detection regularization), at two timesteps $t$ and $t+1$, the OOD detection objective difference conflicts with OOD generalization difference. In addition, the generalization error difference over time is not only negatively correlated with OOD detection loss that the model minimizes, it also negatively correlated to the Fisher information of the network parameter under $\mathbb{P}_{out}^{t,cov}$. The OOD generalization error at $t+1$ and $t$ is positively correlated with confidence difference over the same period. Similar to~\cite{zhang2024best} our theorem is applicable for all MSP-based OOD detectors. The inherent motivation of OOD detection methods lies in minimizing the OOD detection loss in $\mathbb{P}_{out}^{t,sem}$ under test data, regardless of the training strategies used.

%% file: sec/5_experiments.tex
\section{Experiments}
\label{sec:experiments}
We evaluate T-QPM on temporally evolving benchmarks designed to test semantic OOD detection and covariate shift robustness under continuously shifting distributions.

\noindent\textbf{Datasets.}
For ID data, we use three temporal benchmarks: CLEAR100~\cite{lin2022clearbenchmarkcontinuallearning}, CLEAR10~\cite{lin2022clearbenchmarkcontinuallearning}, and Core50~\cite{lomonaco2017core50newdatasetbenchmark}. CLEAR100 and CLEAR10 span 10 temporal buckets, each representing a distinct time period. Core50 consists of 10 sessions captured under varying backgrounds and lighting conditions, providing a complementary setting with more abrupt session-level domain shifts. Since T-QPM operates on image-caption pairs, all semantic OOD datasets are drawn from multimodal sources: COCO~\cite{lin2015microsoftcococommonobjects}, ImageNet-1K-VL-Enriched~\cite{imagenet_captions_hf}, Visual Genome~\cite{krishna2016visualgenomeconnectinglanguage}, Flickr30K~\cite{plummer2016flickr30kentitiescollectingregiontophrase}, and CC12M~\cite{changpinyo2021cc12m}. To assess covariate shift robustness, we generate perturbed variants of each ID test set using Gaussian blur and JPEG compression corruptions.

\noindent\textbf{Training Procedure.}
T-QPM is trained sequentially across timesteps, where
at each timestep the model receives image-caption pairs from the current tem-
poral distribution. In the CLEAR100/CLEAR10 setting, training proceeds from timestep 1 through timestep 10, with each bucket representing a progressively
drifted visual distribution. For Core50, the model is trained across 10 recording
sessions in order of acquisition. At each timestep, the model is updated using
the current ID data while the projection module adapts its interference weights

\noindent\textbf{Model Architectures and Optimization.} T-QPM is built on top of two frozen CLIP backbones, ViT-B/16 and ViT-B/32. All experiments use a learning rate of $3 \times 10^{-3}$ with 5 epochs per timestep. The frozen backbone ensures
that pretrained vision-language representations are preserved, while only the
projection module is updated to adapt to temporal distribution shift.

\noindent\textbf{Evaluation Protocol.}
We report results (average over 3 trials) at representative early and late timesteps, $t{=}2$ and $t{=}8$, to capture model behavior before and after substantial temporal drift(other timesteps are re-
ported in the SM). We report FPR95 and AUROC as threshold-independent
measures of OOD detection quality, alongside ID clean accuracy on the un-
perturbed test set and ID corrupted accuracy on blur- and compression-
degraded variants. 

\subsection{Comparison with Existing Methods}

Tables~\ref{tab:comparison_methods} and~\ref{tab:ood_results_early_late_vit32_full} report OOD detection performance (FPR95$\downarrow$/AUROC$\uparrow$) across three temporally evolving ID datasets and five semantic OOD datasets at early ($t{=}2$) and late ($t{=}8$) timesteps using ViT-B/16 and ViT-B/32 backbones, respectively. Across all datasets, OOD benchmarks, temporal stages, and backbone architectures, T-QPM consistently outperforms existing VLM-based OOD detection methods, including MCM, LoCoOp, and DPM, achieving the lowest FPR95 and highest AUROC.\\
The most significant improvements are observed on the challenging CLEAR100 benchmark, where temporal drift has the greatest impact on OOD detection performance. Using the ViT-B/16 backbone, T-QPM achieves an FPR95/AUROC of 17.42\%/96.66\% on COCO at the early timestep and 19.37/96.07 on Visual Genome at the late timestep, consistently outperforming all competing methods. The best overall performance is obtained on the CLEAR10--CC12M benchmark, where T-QPM achieves an FPR95 of 0.46\% and an AUROC of 99.99\% at the early timestep, while maintaining similarly strong performance at the late timestep (1.59\%/99.57\%). Comparable improvements are also observed on Core50, where T-QPM consistently ranks first across all semantic OOD datasets, demonstrating that the proposed framework generalizes well across both gradual temporal drift and abrupt session-level domain shifts. The same trends are observed for the ViT-B/32 backbone. Although all methods experience a modest reduction in absolute performance compared with ViT-B/16, T-QPM consistently achieves the best overall performance across all ID and OOD dataset combinations. Furthermore, while MCM, LoCoOp, and DPM exhibit noticeable degradation from the early to the late timestep as temporal drift accumulates, T-QPM shows substantially smaller degradation, indicating that the proposed temporal quadruple-pattern matching framework provides a more stable and robust multimodal OOD detector under continuously evolving data distributions.
\begin{table*}[t!]
\centering
\caption{Comparison with existing VLM-based OOD detection methods using ViT-B/16. Each cell reports FPR95$\downarrow$/AUROC$\uparrow$ at early ($t{=}2$) and late ($t{=}8$) timesteps.}
\label{tab:comparison_methods}
\scriptsize{
\setlength{\tabcolsep}{7.5pt}
\renewcommand{\arraystretch}{1.05}
\resizebox{\textwidth}{!}{%
\begin{tabular}{llcccccc}
\toprule
OOD & Method
& \multicolumn{2}{c}{CLEAR100}
& \multicolumn{2}{c}{CLEAR10}
& \multicolumn{2}{c}{Core50}\\
\cmidrule(lr){3-4}\cmidrule(lr){5-6}\cmidrule(lr){7-8}
& & Early & Late & Early & Late & Early & Late\\
\midrule

\multirow{4}{*}{COCO}
& MCM    & 31.10/92.80 & 39.30/90.40 & 24.80/94.70 & 32.00/92.50 & 32.80/92.30 & 41.00/90.10\\
& LoCoOp & 24.80/94.90 & 30.80/93.10 & 19.40/96.40 & 24.40/94.80 & 26.40/94.50 & 32.40/92.70\\
& DPM    & 41.53/88.16 & 46.73/85.55 &  8.61/97.28 &  9.66/96.38 & 16.40/95.60 & 24.10/93.20\\
& T-QPM  & \textbf{17.42/96.66} & \textbf{20.51/95.77} & \textbf{0.89/99.66} & \textbf{1.20/99.38} & \textbf{6.20/98.90} & \textbf{9.80/97.60}\\

\midrule
\multirow{4}{*}{IN-1K}
& MCM    & 34.40/92.00 & 42.60/89.60 & 27.60/93.90 & 34.80/91.70 & 36.70/91.00 & 44.90/88.80\\
& LoCoOp & 27.50/94.10 & 33.50/92.30 & 22.10/95.60 & 27.10/94.00 & 30.20/93.40 & 36.20/91.60\\
& DPM    & 17.58/95.74 & 22.48/94.41 &  9.51/98.48 & 14.70/94.74 & 11.20/97.80 & 17.30/95.40\\
& T-QPM  & \textbf{5.97/98.79} & \textbf{7.24/98.59} & \textbf{3.65/99.16} & \textbf{5.20/98.95} & \textbf{4.10/99.10} & \textbf{6.90/98.70}\\

\midrule
\multirow{4}{*}{Flickr30K}
& MCM    & 28.60/93.70 & 36.80/91.30 & 23.20/95.10 & 30.40/92.90 & 30.40/93.00 & 38.60/90.60\\
& LoCoOp & 22.60/95.60 & 28.60/93.80 & 18.30/96.60 & 23.30/95.40 & 24.60/95.00 & 30.60/93.20\\
& DPM    & 21.62/94.26 & 26.10/92.95 &  6.03/98.28 &  8.31/97.68 & 14.10/96.10 & 21.50/93.80\\
& T-QPM  & \textbf{7.96/98.20} & \textbf{8.95/98.06} & \textbf{1.63/99.65} & \textbf{2.60/99.07} & \textbf{5.10/99.00} & \textbf{8.20/98.10}\\

\midrule
\multirow{4}{*}{CC12M}
& MCM    & 32.00/92.50 & 40.20/90.10 & 26.90/94.20 & 34.10/92.00 & 34.10/91.80 & 42.30/89.60\\
& LoCoOp & 25.90/94.70 & 31.90/92.90 & 21.20/96.00 & 26.20/94.40 & 27.80/94.00 & 33.80/92.20\\
& DPM    & 10.23/97.74 & 12.51/97.11 &  7.35/98.53 &  9.29/98.57 & 10.30/97.80 & 15.20/95.60\\
& T-QPM  & \textbf{2.56/99.48} & \textbf{3.63/99.28} & \textbf{0.46/99.99} & \textbf{1.59/99.57} & \textbf{2.70/99.50} & \textbf{5.10/99.00}\\

\midrule
\multirow{4}{*}{VG}
& MCM    & 27.30/94.10 & 35.50/91.70 & 22.30/95.50 & 29.50/93.30 & 29.00/93.20 & 37.20/90.80\\
& LoCoOp & 21.70/96.00 & 27.70/94.20 & 17.40/97.10 & 22.40/95.50 & 23.50/95.30 & 29.50/93.50\\
& DPM    & 44.42/87.46 & 50.16/84.37 & 23.06/93.18 & 38.66/90.52 & 22.80/94.10 & 34.60/89.80\\
& T-QPM  & \textbf{16.18/96.57} & \textbf{19.37/96.07} & \textbf{11.85/97.49} & \textbf{13.75/98.12} & \textbf{10.40/97.60} & \textbf{16.90/95.80}\\

\bottomrule
\end{tabular}}}
\end{table*}
\begin{table}[h!]
\centering
\caption{Ablation study of loss components on CLEAR100 as ID and COCO as OOD using ViT-B/16.}
\label{tab:loss_ablation}
\small
\setlength{\tabcolsep}{3pt}
\renewcommand{\arraystretch}{1.05}
\resizebox{\linewidth}{!}{%
\begin{tabular}{cccc cc l}
\toprule
$\mathcal{L}_\text{ID}$ & $\mathcal{L}_\text{COV}$ & $\mathcal{L}_\text{TEMP}$ & $\mathcal{L}_\text{OOD}$ & AUROC $\uparrow$ & FPR95 $\downarrow$ & Note \\
\midrule
\checkmark & $\times$ & $\times$ & $\times$ & 91.91 & 24.12 & $\mathcal{L}_\text{ID}$ only \\
$\times$ & \checkmark & $\times$ & $\times$ & 85.51 & 39.22 & $\mathcal{L}_\text{COV}$ only \\
$\times$ & $\times$ & \checkmark & $\times$ & 89.21 & 32.92 & $\mathcal{L}_\text{TEMP}$ only \\
$\times$ & $\times$ & $\times$ & \checkmark & 88.01 & 31.02 & $\mathcal{L}_\text{OOD}$ only \\
\midrule
$\times$ & \checkmark & \checkmark & \checkmark & 93.01 & 21.72 & w/o $\mathcal{L}_\text{ID}$ \\
\checkmark & $\times$ & \checkmark & \checkmark & 96.41 & 19.62 & w/o $\mathcal{L}_\text{COV}$ \\
\checkmark & \checkmark & $\times$ & \checkmark & 95.71 & 24.92 & w/o $\mathcal{L}_\text{TEMP}$ \\
\checkmark & \checkmark & \checkmark & $\times$ & 96.91 & 16.82 & w/o $\mathcal{L}_\text{OOD}$ \\
\midrule
\checkmark & \checkmark & \checkmark & \checkmark & \textbf{97.81} & \textbf{10.28} & Full model \\
\bottomrule
\end{tabular}}
\end{table}
\begin{table}[h!]
\centering
\caption{Sensitivity analysis of $\lambda_{\rm cov}$ and $\lambda_{\rm temp}$ on CLEAR100/COCO using ViT-B/16 at $t{=}2$. Bold denotes the chosen setting.}
\label{tab:lambda_sweep}
\small
\setlength{\tabcolsep}{4pt}
\renewcommand{\arraystretch}{1.00}
\begin{tabular}{ccc@{\hspace{10pt}}ccc}
\toprule
$\lambda_{\rm cov}$ & AUROC & FPR95 & $\lambda_{\rm temp}$ & AUROC & FPR95 \\
\midrule
0.10 & 95.20 & 18.40 & 0.25 & 94.90 & 21.30 \\
0.25 & 96.80 & 13.60 & 0.50 & 96.40 & 14.70 \\
\textbf{0.50} & \textbf{97.81} & \textbf{10.28} & \textbf{1.00} & \textbf{97.81} & \textbf{10.28} \\
1.00 & 97.10 & 12.50 & 2.00 & 97.20 & 11.90 \\
2.00 & 95.60 & 17.80 & 5.00 & 95.10 & 19.60 \\
\bottomrule
\end{tabular}
\end{table}
\begin{figure}[b!]
    \centering
    \begin{subfigure}{0.7\linewidth}
        \includegraphics[width=\linewidth]{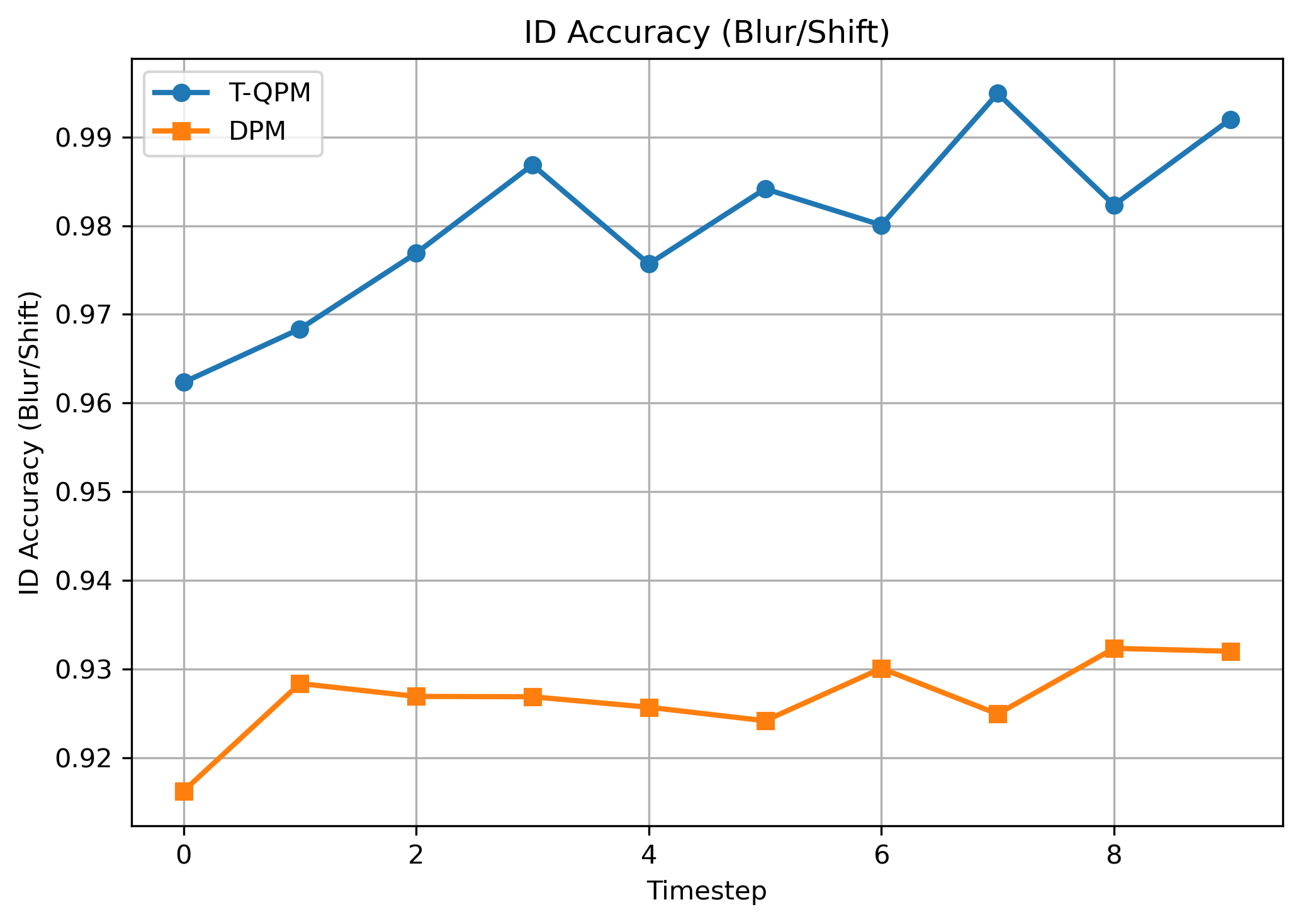}
   \end{subfigure}
      \caption{JPEG-compressed CLEAR100 test sets across all timesteps. Results are averaged over 3 trials.}
    \label{fig:acc_jpeg_blur}
\end{figure}

\subsection{Robustness against covariate shifted data}
Figure~\ref{fig:acc_jpeg_blur} report 
corrupted test sets across all 10 timesteps of CLEAR100, with COCO as the semantic OOD  dataset. 
Under JPEG compression corruption (Figure~\ref{fig:acc_jpeg}), T-QPM demonstrates an even more 
pronounced advantage. 
More strikingly, on JPEG-corrupted inputs, T-QPM exhibits a consistent upward trajectory 
across all timesteps, reaching $\sim$0.991 at $t{=}7$---while DPM remains nearly flat in  the 0.920--0.932 range throughout. This $\sim$5--6\% sustained gap under JPEG shift 
suggests that T-QPM's interference-based scoring mechanism is particularly robust to high-frequency compression artifacts, which tend to destabilize standard softmax-based confidence estimates. 
More experiments on robustness are provided in SM.

\begin{figure*}[t!]
    \centering
    \begin{subfigure}{0.32\linewidth}
        \centering
        \includegraphics[width=\linewidth]{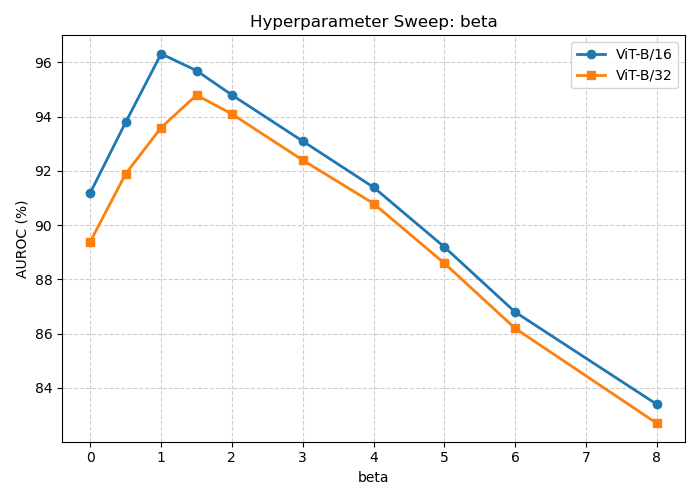}
        \label{fig:sweep_beta}
    \end{subfigure}
    \hfill
    \begin{subfigure}{0.32\linewidth}
        \centering
        \includegraphics[width=\linewidth]{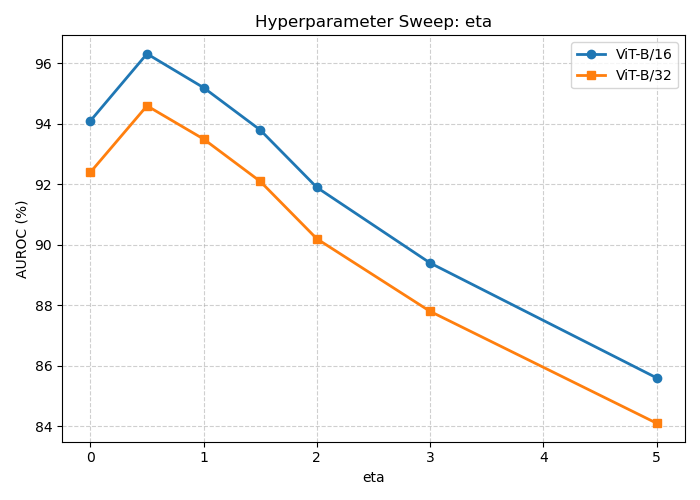}
        \label{fig:sweep_eta}
    \end{subfigure}
    \hfill
    \begin{subfigure}{0.32\linewidth}
        \centering
        \includegraphics[width=\linewidth]{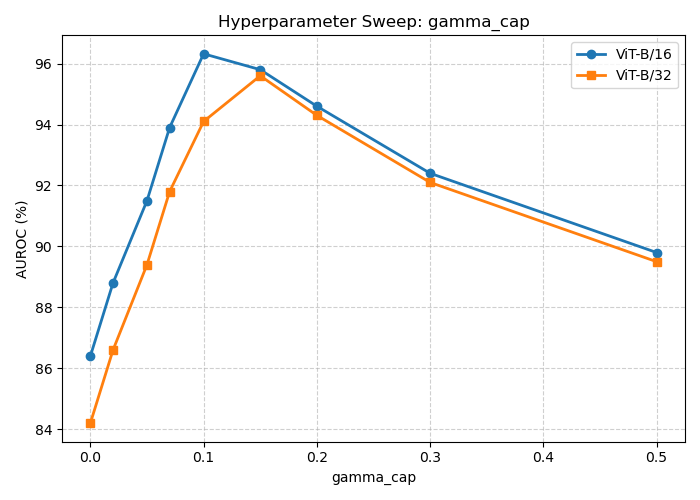}
        \label{fig:sweep_gamma}
    \end{subfigure}
    \caption{Hyperparameter sweeps for $\beta$ (left), $\eta$ (center), and $\gamma_{\mathrm{cap}}$ (right).}
    \label{fig:sweeps}
\end{figure*}

\begin{table*}[h!]
\centering
  \caption{OOD detection results across 3 ID datasets and 5 OOD datasets using ViT-32 architecture.
  Each cell shows FPR95$\downarrow$/AUROC$\uparrow$.
  Early ($t{=}2$) and late ($t{=}8$) times.}
  \label{tab:ood_results_early_late_vit32_full}
  \centering
\scriptsize{
  \setlength{\tabcolsep}{8.5pt}
  \renewcommand{\arraystretch}{1.05}
  \resizebox{\textwidth}{!}{%
  \begin{tabular}{llcccccc}
    \toprule
    OOD & Method
      & \multicolumn{2}{c}{CLEAR100}
      & \multicolumn{2}{c}{CLEAR10}
      & \multicolumn{2}{c}{Core50} \\
    \cmidrule(lr){3-4} \cmidrule(lr){5-6} \cmidrule(lr){7-8}
    & & Early & Late & Early & Late & Early & Late \\
    \midrule

    \multirow{4}{*}{COCO}
      & MCM    & 33.10/92.30 & 41.30/89.90 & 26.80/94.20 & 34.00/92.00 & 34.80/91.80 & 43.00/89.60 \\
      & LoCoOp & 26.80/94.40 & 32.80/92.60 & 21.40/95.90 & 26.40/94.30 & 28.40/94.00 & 34.40/92.20 \\
      & DPM    & 47.73/87.30 & 55.11/84.30 & 9.89/96.40 & 11.33/95.10 & 18.86/94.80 & 28.44/92.00 \\
      & T-QPM  & \textbf{19.14/96.20} & \textbf{22.96/95.00} & \textbf{0.88/99.20} & \textbf{1.34/98.60} & \textbf{6.82/98.50} & \textbf{10.98/96.90} \\
    \midrule

    \multirow{4}{*}{IN-1K}
      & MCM    & 36.40/91.50 & 44.60/89.10 & 29.60/93.40 & 36.80/91.20 & 38.70/90.50 & 46.90/88.30 \\
      & LoCoOp & 29.50/93.60 & 35.50/91.80 & 24.10/95.10 & 29.10/93.50 & 32.20/92.90 & 38.20/91.10 \\
      & DPM    & 20.13/94.90 & 26.43/93.20 & 10.93/98.00 & 17.35/93.50 & 12.88/97.00 & 20.41/94.20 \\
      & T-QPM  & \textbf{6.49/98.30} & \textbf{8.06/97.80} & \textbf{3.96/98.70} & \textbf{5.82/98.20} & \textbf{4.51/98.70} & \textbf{7.73/98.00} \\
    \midrule

    \multirow{4}{*}{Flickr30K}
      & MCM    & 30.60/93.20 & 38.80/90.80 & 25.20/94.60 & 32.40/92.40 & 32.40/92.50 & 40.60/90.10 \\
      & LoCoOp & 24.60/95.10 & 30.60/93.30 & 20.30/96.10 & 25.30/94.90 & 26.60/94.50 & 32.60/92.70 \\
      & DPM    & 24.84/93.40 & 30.80/91.70 & 6.90/97.40 & 9.79/96.40 & 16.22/95.30 & 25.37/92.60 \\
      & T-QPM  & \textbf{8.69/97.80} & \textbf{9.97/97.30} & \textbf{1.76/99.20} & \textbf{2.91/98.30} & \textbf{5.61/98.60} & \textbf{9.18/97.40} \\
    \midrule

    \multirow{4}{*}{CC12M}
      & MCM    & 34.00/92.00 & 42.20/89.60 & 28.90/93.70 & 36.10/91.50 & 36.10/91.30 & 44.30/89.10 \\
      & LoCoOp & 27.90/94.20 & 33.90/92.40 & 23.20/95.50 & 28.20/93.90 & 29.80/93.50 & 35.80/91.70 \\
      & DPM    & 11.73/96.90 & 14.75/95.90 & 8.39/97.70 & 10.86/97.30 & 11.84/97.00 & 17.94/94.40 \\
      & T-QPM  & \textbf{2.75/99.00} & \textbf{4.03/98.50} & \textbf{0.44/99.50} & \textbf{1.68/98.80} & \textbf{2.97/99.10} & \textbf{5.71/98.30} \\
    \midrule

    \multirow{4}{*}{VG}
      & MCM    & 29.30/93.60 & 37.50/91.20 & 24.30/95.00 & 31.50/92.80 & 31.00/92.70 & 39.20/90.30 \\
      & LoCoOp & 23.70/95.50 & 29.70/93.70 & 19.40/96.60 & 24.40/95.00 & 25.50/94.80 & 31.50/93.00 \\
      & DPM    & 51.06/86.60 & 59.12/83.50 & 26.45/92.30 & 45.55/89.30 & 26.22/93.30 & 40.83/88.60 \\
      & T-QPM  & \textbf{17.71/96.30} & \textbf{21.62/95.30} & \textbf{12.98/97.00} & \textbf{15.34/97.40} & \textbf{11.44/97.20} & \textbf{18.93/95.10} \\

    \bottomrule
  \end{tabular}%
  }}

\end{table*}
\subsection{Ablation Study \& Hyperparameter Sensitivity}

Table~\ref{tab:loss_ablation} examines the contribution of each loss component on CLEAR100 with COCO as OOD using ViT-B/16. When used in isolation, no single loss achieves the full model performance. The leave-one-out results show that every component contributes meaningfully. Removing $\mathcal{L}_{\mathrm{COV}}$ or $\mathcal{L}_{\mathrm{TEMP}}$ causes the largest degradation, confirming the importance of covariate consistency and temporal loss. The complete objective achieves the best AUROC \& FPR95.
Figure~\ref{fig:sweeps} reports AUROC sweeps over the fusion parameters $\beta$, $\eta$, and $\gamma_{\mathrm{cap}}$. All three parameters exhibit stable operating regions, suggesting that T-QPM does not require overly precise tuning. We further analyze the loss weights $\lambda_{\mathrm{cov}}$ and $\lambda_{\mathrm{temp}}$ in Table~\ref{tab:lambda_sweep}. Both show broad stable plateaus, and performance degrades only at extreme values.

%% file: sec/6_conclusion.tex
\section{Conclusion}

T-QPM consistently outperforms baselines across all evaluated settings, confirming that four-way scoring over image-caption pairs provides a more powerful and temporally robust OOD detection signal than dual-pattern matching alone. Notably, the performance gap widens as temporal drift accumulates, suggesting that caption-aware scoring is particularly effective at maintaining stable decision boundaries under evolving ID data. 
Our observations suggest that caption quality and linguistic diversity of the OOD source are significant yet largely unexplored factors in multimodal OOD detection. The proposed T-QPM provides a temporally-aware quadruple matching framework for multimodal OOD detection under continuously shifting distributions. By building on frozen CLIP backbones and introducing a caption-aware scoring mechanism, T-QPM jointly leverages visual and linguistic ID information to produce reliable OOD detection signals across temporal benchmarks, establishing a strong and principled baseline for caption-aware, temporally robust OOD detection in open-world multi-modal settings and motivating further investigation into multimodal-continual learning in dynamic environments.
{\bf Future work} includes extending T-QPM by an online caption-prompt mechanism integrated into a multi-agent system that closes the loop for OOD detection reasoning.

%% file: sec/Acknowledgement.tex
\section*{Acknowledgements}
This work has been partially supported by NSF CAREER CCF-2451457. The findings are those of the authors only and do not represent any position of these funding bodies.

%% file: sec/X_supp.tex

\input{sec/2_extensive-related_work}

\section{Theoretical Proofs}

\begin{lemma}
At time steps $t$ and $t{+}1$, if $H(p(y_t|x_t))\leq H(p(y_{t+1}|x_{t+1}))$ then
\begin{align*}Conf_{t}&=\max\limits_{y_t\in \mathcal{Y}_t} p(y_t|x_t)\\
&\geq \max\limits_{y_{t+1}\in \mathcal{Y}_{t+1}} p(y_{t+1}|x_{t+1})=Conf_{t+1}.\end{align*}
\end{lemma}
{\bf Proof:} For $K$ classes at both time $t$ and $t{+}1$, denote $p^*_t:=\max_{y_t\in \mathcal{Y}_t} p(y_t|x_t)$ and $p^*_{t+1}:=\max_{y_{t+1}\in \mathcal{Y}_{t+1}} p(y_{t+1}|x_{t+1})$. Suppose $p^*_t=P(y_t=k_1|x_t)$ and $p^*_{t+1}=P(y_t=k_2|x_{t+1})$. Now set $p_t=(p^*_t, 1{-}p^*_t)$, where $1{-}p^*_t$ is split among $\{1,\ldots,K\}\setminus\{k_1\}$ and $1{-}p^*_{t+1}$ is split among $\{1,\ldots,K\}\setminus\{k_2\}$. This approximates the entropy as
\begin{equation}\label{eq:1}
H(p_t)=-p^*_t\log p^*_t-\sum\limits_{i\in \{1,\ldots,K\}\setminus\{k_1\}} p_{it}\log p_{it},
\end{equation}
where $p_{it}=\frac{1-p^*_t}{K-1}$. And (\ref{eq:1}) is simplified as
\begin{equation}\label{eq:2}
H(p_t)=-p^*_t\log p^*_t-(1-p^*_t)\log \frac{1-p^*_t}{K-1}.
\end{equation}
Equivalently
\begin{equation}\label{eq:3}
H(p_{t+1})=-p^*_{t+1}\log p^*_{t+1}-(1-p^*_{t+1})\log \frac{1-p^*_{t+1}}{K-1}.
\end{equation}
Because $H(p_t)\leq H(p_{t+1})$ and from (\ref{eq:2}) and (\ref{eq:3}), we implies that $p^*_t\geq p^*_{t+1}$.

\begin{lemma} (Theorem 1, (\cite{zhang2024best}))
The generalization error at time step $t$, $GErr_t$, is standard cross entropy loss for hypothesis $f\in\mathcal{F}$ under covariant shift $\mathbb{P}^{cov}$. $GErr_t$ is lower bounded by
\begin{align}
 GErr_t (f) \geq  -\frac{1}{2\kappa}\mathbb{E}_{\mathbb{P}_{out}^{t,sem}}\sqrt{\frac{1}{2}(\mathcal{L}_{reg}(f)-\log K)} \\\nonumber
 - \frac{1}{2\kappa} d_{\mathcal{F}}(\mathbb{P}_{out}^{t,cov}, \mathbb{P}_{out}^{t,sem})+C_t+\mathbb{E}_{\mathbb{P}_{out}^{t,cov}}H(p(y_t|x_t))
\end{align}
\end{lemma}
where $C_t$ is constant.

\begin{lemma}\label{lemma2} (Lemma~1, (\cite{zhang2024best}))
For any $f \in \mathcal{F}$, we have
\begin{align}
\mathbb{E}_{\mathbb{P}_{out}^{t,cov}} TV(F_f\|\mathcal{U})\leq &\mathbb{E}_{\mathbb{P}_{out}^{t,sem}} TV(F_f\|\mathcal{U})\\\nonumber
&+ d_{\mathcal{F}}(\mathbb{P}_{out}^{t,cov}, \mathbb{P}_{out}^{t,sem})+\lambda,
\end{align}
where $\lambda$ is a constant independent of $f$. $\mathcal{U}$ is the $K$-classes uniform distribution. $\mathbb{P}_{out}^{t,cov}$ is the covariate-shifted OOD distribution at time $t$. $\mathbb{P}_{out}^{t,sem}$ is the semantic OOD distribution at time $t$.
\end{lemma}

\begin{lemma}\label{lemma3} (Lemma~3, (\cite{zhang2024best}))
Denote the OOD detection loss used for MSP detectors as $\mathcal{L}_{reg}$, then we have
\begin{equation}
\mathbb{E}_{\mathbb{P}_{out}^{t,sem}}  \left(TV(F_f\|\mathcal{U})\right)\leq \mathbb{E}_{\mathbb{P}_{out}^{t,sem}}   \sqrt{\frac{1}{2}(\mathcal{L}_{reg}(f)-\log K)}.
\end{equation}
\end{lemma}

\begin{lemma} The generalization error at time step $t$, $GErr_t$, is standard cross entropy loss for hypothesis $f\in\mathcal{F}$ under covariant shift $\mathbb{P}^{cov}$. $GErr_t$ is upper bounded by
\begin{align}
  GErr_{t} (f)& \leq \frac{\log e}{2} \mathbb{E}_{\mathbb{P}_{out}^{t,sem}} \sqrt{\frac{1}{2}(\mathcal{L}_{reg}(f)-\log K)} \nonumber\\
  &\quad + \frac{\log e}{2}  d_{\mathcal{F}}(\mathbb{P}_{out}^{t,cov}, \mathbb{P}_{out}^{t,sem})\\
  &+ C_t+ \frac{\log e}{2}\mathbb{E}_{\mathbb{P}_{out}^{t,cov}}\left(\mathcal{X}^2 (p(y_{t}|x_{t})\|F_f(x_{t}))\right) \nonumber\\
  &\quad +H(p(y_{t}|x_{t}))
\end{align}
where $C_t$ is constant.
\end{lemma}

{\bf Proof:}
\begin{align}
  GErr_{t} (f)&:=\mathbb{E}_{\mathbb{P}_{out}^{t,cov}}\mathcal{L}_{CE}(f(x_{t},y_{t}))\nonumber\\
  & =\mathbb{E}_{\mathbb{P}_{out}^{t,cov}} KL(p(y_{t}|x_{t})\| F_f(x_{t}))+H(p(y_{t}|x_{t}))\nonumber\\
  & \leq \frac{\log e}{2}\mathbb{E}_{\mathbb{P}_{out}^{t,cov}}  (TV(p(y_{t}|x_{t})\|F_f(x_{t})) \nonumber\\
  &\quad +\mathcal{X}^2 (p(y_{t}|x_{t})\|F_f(x_{t}))) +H(p(y_{t}|x_{t}))\nonumber\\
&\leq \frac{\log e}{2}\mathbb{E}_{\mathbb{P}_{out}^{t,cov}}  \left(TV(p(y_{t}|x_{t})\|\mathcal{U})\right.\nonumber\\
&\qquad\left.+ TV(F_f(x_{t})\|\mathcal{U})\right)\nonumber\\
&+\mathbb{E}_{\mathbb{P}_{out}^{t,cov}} \left(\mathcal{X}^2 (p(y_{t}|x_{t})\|F_f(x_{t}))\right) \nonumber\\
&\quad +\mathbb{E}_{\mathbb{P}_{out}^{t,cov}}H(p(y_{t}|x_{t}))
\end{align}
where from ~\cite{sason2016f}, we have
$$
\mathcal{X}^2 (P\|Q)+1=\int \frac{P^2}{Q} d\mu
$$
and from~\cite{nishiyama2020relations} we have
$$
KL(P\|Q)\leq \frac{1}{2} \left(TV(P\|Q) +\mathcal{X}^2 (P\|Q)\right)\log e
$$
From Lemma~\ref{lemma2} above we have
\begin{align}
  GErr_{t} (f)&\leq \frac{\log e}{2}\mathbb{E}_{\mathbb{P}_{out}^{t,cov}}  \left(TV(p(y_{t}|x_{t})\|\mathcal{U})\right) \nonumber\\
 &\quad + \frac{\log e}{2} \mathbb{E}_{\mathbb{P}_{out}^{t,sem}} TV(F_f\|\mathcal{U})\nonumber\\
 & +\frac{\log e}{2}  d_{\mathcal{F}}(\mathbb{P}_{out}^{t,cov}, \mathbb{P}_{out}^{t,sem})\\
 & +\frac{\log e}{2}\lambda + \frac{\log e}{2}\mathbb{E}_{\mathbb{P}_{out}^{t,cov}}\left(\mathcal{X}^2 (p(y_{t}|x_{t})\|F_f(x_{t}))\right) \nonumber\\
 &\quad + \mathbb{E}_{\mathbb{P}_{out}^{t,cov}}H(p(y_{t}|x_{t}))
  \end{align}
From Lemma~\ref{lemma3} above we have
\begin{align}
  GErr_{t} (f)&\leq \frac{\log e}{2}\mathbb{E}_{\mathbb{P}_{out}^{t,cov}}  \left(TV(p(y_{t}|x_{t})\|\mathcal{U})\right) \nonumber\\
 &\quad + \frac{\log e}{2} \mathbb{E}_{\mathbb{P}_{out}^{t,sem}}   \sqrt{\frac{1}{2}(\mathcal{L}_{reg}(f)-\log K)} \nonumber \\
 &+ \frac{\log e}{2}  d_{\mathcal{F}}(\mathbb{P}_{out}^{t,cov}, \mathbb{P}_{out}^{t,sem})+\frac{\log e}{2}\lambda \\
 &+ \frac{\log e}{2}\mathbb{E}_{\mathbb{P}_{out}^{t,cov}}\left(\mathcal{X}^2 (p(y_{t}|x_{t})\|F_f(x_{t}))\right) \nonumber\\
 &\quad + \mathbb{E}_{\mathbb{P}_{out}^{t,cov}}H(p(y_{t}|x_{t}))
  \end{align}
since at each time $t$, $\mathbb{E}_{\mathbb{P}_{out}^{t,cov}}  \left(TV(p(y_{t}|x_{t})\|\mathcal{U})\right)$ is constant, we upper bound $GErr_{t} (f)$ as
\begin{align}
  GErr_{t} (f)& \leq \frac{\log e}{2} \mathbb{E}_{\mathbb{P}_{out}^{t,sem}} \sqrt{\frac{1}{2}(\mathcal{L}_{reg}(f)-\log K)} \nonumber\\
  &\quad + \frac{\log e}{2}  d_{\mathcal{F}}(\mathbb{P}_{out}^{t,cov}, \mathbb{P}_{out}^{t,sem})\\
  &+ C_t+ \frac{\log e}{2}\mathbb{E}_{\mathbb{P}_{out}^{t,cov}}\left(\mathcal{X}^2 (p(y_{t}|x_{t})\|F_f(x_{t}))\right) \nonumber\\
  &\quad +\mathbb{E}_{\mathbb{P}_{out}^{t,cov}}H(p(y_{t}|x_{t}))
  \end{align}

\begin{lemma}\label{lemma4} Under the assumption ${\bf [A2]}$ and regularity condition on $F^{\theta_1}_f$, we have
\begin{equation}
\mathbb{E}_{\mathbb{P}_{out}^{t,cov}}\!\left(\mathcal{X}^2 (p(y_{t}|x_{t})\|F_f(x_{t}))\right) \leq \delta_t^2 \;\mathbb{E}_{\mathbb{P}_{out}^{t,cov}}\!\left(I_F(\theta_2)\right) + B_t,
\end{equation}
where $I_F(\theta_2)$ is Fisher information and $B_t$ is constant.
The key part of this conjecture is developed based on
\begin{align}
\mathbb{E}_{\mathbb{P}_{out}^{t,cov}}\!\left(\mathcal{X}^2 (p(y_{t}|x_{t})\|F_f(x_{t}))\right) &= (\theta_1{-}\theta_2)^2  \mathbb{E}_{\mathbb{P}_{out}^{t,cov}}\!\left(I_F(\theta_2)\right) \nonumber\\
&\quad +o(\theta_1{-}\theta_2)^2,
\end{align}
where $\theta_1$ approximately vanishes.
\end{lemma}

Because inverse of entropy can be used as a confidence score to gauge the likelihood of a prediction being correct, we assume:\\
{\bf [A3]} There exist a constant (say $Z_t$), such that
\begin{align}
\mathbb{E}_{\mathbb{P}_{out}^{t+1,cov}}H(p(y_{t+1}|x_{t+1}))&-\mathbb{E}_{\mathbb{P}_{out}^{t,cov}}H(p(y_{t}|x_{t})) \nonumber\\
&\geq Z_t+ Conf_{t}- Conf_{t+1}
\end{align}

  \begin{theorem} {\bf (Main Theorem)} Let $\mathbb{P}^{t,cov}$ and $\mathbb{P}^{t,sem}$ be the covariate-shifted OOD and semantic OOD distributions. Denote $GErr_{t+1}(f)$ the generalization error at time $t+1$.
  Then at two time steps $t$ and $t+1$ and under assumptions {\bf [A1]}, {\bf [A2]}, and {\bf [A3]}, we have
  \begin{align}\label{thm1:main-eq}
  GErr_{t+1}(f)-GErr_{t}(f) &\geq - \tilde{\kappa}\; \Delta_{t\rightarrow t+1}^{cov,sem} - \tilde{\kappa}\; \Xi_{t\rightarrow t+1}^{sem} \nonumber\\
  &\quad - \overline{\delta}_t^2 \;\mathbb{E}_{\mathbb{P}_{out}^{t,cov}}\left(I_F(\theta_2)\right) \nonumber\\
 &+C_{t\rightarrow t+1} + Conf_{t} - Conf_{t+1},
  \end{align}
  where
  \begin{align*}\Delta_{t\rightarrow t+1}^{cov,sem}&:=d_{\mathcal{F}}(\mathbb{P}_{out}^{t+1, cov}, \mathbb{P}_{out}^{t+1,sem})\\
  &\quad + d_{\mathcal{F}}(\mathbb{P}_{out}^{t, cov}, \mathbb{P}_{out}^{t,sem})\end{align*}
  and
  \begin{align*}
  \Xi_{t\rightarrow t+1}^{sem}&:= \mathbb{E}_{\mathbb{P}_{out}^{t+1,sem}}\sqrt{\frac{1}{2}(\mathcal{L}_{reg}(f)-\log K)}\\
  &\quad + \mathbb{E}_{\mathbb{P}_{out}^{t,sem}} \sqrt{\frac{1}{2}(\mathcal{L}_{reg}(f)-\log K)}.
  \end{align*}
  And $C_{t\rightarrow t+1}=C_{t+1}-C_t+B_t+Z_t$ and $\delta_t$ are constants and $\overline{\delta}_t^2=\frac{\log e}{2}\delta_t^2$.

  {\bf Proof:} Recall the definition of $GErr_{t}(f)$:
  \begin{align}\label{thm1:eq1}
  GErr_{t+1}(f)-GErr_{t}(f) \nonumber \\
  \geq -\frac{1}{2\kappa}\mathbb{E}_{\mathbb{P}_{out}^{t+1,sem}}\sqrt{\frac{1}{2}(\mathcal{L}_{reg}(f)-\log K)} \nonumber\\
  \quad - \frac{1}{2\kappa} d_{\mathcal{F}}(\mathbb{P}_{out}^{t+1, cov}, \mathbb{P}_{out}^{t+1,sem})\nonumber\\
  -\frac{\log e}{2} \mathbb{E}_{\mathbb{P}_{out}^{t,sem}} \sqrt{\frac{1}{2}(\mathcal{L}_{reg}(f)-\log K)} \nonumber\\
  \quad -\frac{\log e}{2}  d_{\mathcal{F}}(\mathbb{P}_{out}^{t,cov}, \mathbb{P}_{out}^{t,sem}) \nonumber\\
   - \frac{\log e}{2}\mathbb{E}_{\mathbb{P}_{out}^{t,cov}}\left(\mathcal{X}^2 (p(y_{t}|x_{t})\|F_f(x_{t}))\right)\nonumber \\
 +( C_{t+1}-C_t) \nonumber\\
 + (\mathbb{E}_{\mathbb{P}_{out}^{t+1,cov}}H(p(y_{t+1}|x_{t+1}))-\mathbb{E}_{\mathbb{P}_{out}^{t,cov}}H(p(y_{t}|x_{t}))),
  \end{align}\label{SM:eq:delta}
  If we denote
  \begin{align}\Delta_{t\rightarrow t+1}^{cov,sem}&:=d_{\mathcal{F}}(\mathbb{P}_{out}^{t+1, cov}, \mathbb{P}_{out}^{t+1,sem})\nonumber\\
  &\quad + d_{\mathcal{F}}(\mathbb{P}_{out}^{t, cov}, \mathbb{P}_{out}^{t,sem})\end{align}
  and
  \begin{align}\label{SM:eq:Xi}
  \Xi_{t\rightarrow t+1}^{sem}&:= \mathbb{E}_{\mathbb{P}_{out}^{t+1,sem}}\sqrt{\frac{1}{2}(\mathcal{L}_{reg}(f)-\log K)}\nonumber\\
  &\quad + \mathbb{E}_{\mathbb{P}_{out}^{t,sem}} \sqrt{\frac{1}{2}(\mathcal{L}_{reg}(f)-\log K)},
  \end{align}
  then there exists a constant $\tilde{\kappa}\leq \frac{1}{2\kappa}+\frac{\log e}{2}$ such that (\ref{thm1:eq1}) is written as
  {\small
  \begin{align}\label{thm1:eq2}
  GErr_{t+1}(f)-GErr_{t}(f) &\geq - \tilde{\kappa}\; \Delta_{t\rightarrow t+1}^{cov,sem} - \tilde{\kappa}\; \Xi_{t\rightarrow t+1}^{sem} \nonumber\\
  &\quad +C_{t\rightarrow t+1}\nonumber\\
  & - \frac{\log e}{2}\mathbb{E}_{\mathbb{P}_{out}^{t,cov}}\left(\mathcal{X}^2 (p(y_{t}|x_{t})\|F_f(x_{t}))\right)\nonumber \\
 &+\mathbb{E}_{\mathbb{P}_{out}^{t+1,cov}}\!H(p(y_{t+1}|x_{t+1})) \nonumber\\
 &\qquad-\mathbb{E}_{\mathbb{P}_{out}^{t,cov}}H(p(y_{t}|x_{t})),
  \end{align}
  } 
  where $C_{t\rightarrow t+1}=C_{t+1}-C_t$ is constant. Apply the upper bound in Lemma~\ref{lemma4}, we have the lower bound below
  {\small
  \begin{align}\label{thm1:eq3}
  GErr_{t+1}(f)-GErr_{t}(f) &\geq - \tilde{\kappa}\; \Delta_{t\rightarrow t+1}^{cov,sem} - \tilde{\kappa}\; \Xi_{t\rightarrow t+1}^{sem} \nonumber\\
  &\quad - \overline{\delta}_t^2 \;\mathbb{E}_{\mathbb{P}_{out}^{t,cov}}\left(I_F(\theta_2)\right) \nonumber\\
 &+C_{t\rightarrow t+1} \nonumber\\
 &\quad + \mathbb{E}_{\mathbb{P}_{out}^{t+1,cov}}H(p(y_{t+1}|x_{t+1})) \nonumber\\
 &\quad -\mathbb{E}_{\mathbb{P}_{out}^{t,cov}}H(p(y_{t}|x_{t})),
  \end{align}
  } 
  where $C_{t\rightarrow t+1}=C_{t+1}-C_t+B_t$ is constant and $\overline{\delta}_t^2=\frac{\log e}{2}\delta_t^2$. By applying assumption {\bf [A3]}, we conclude the proof.
  \end{theorem}

\section{Pseudocode (Phase I-IV)}




\begin{algorithm}[t]
\caption{Phase I: Text Pattern Construction}
\label{alg:phase1}
\footnotesize
\begin{algorithmic}
\Statex \textbf{Input:}\enspace
  ID class set $\mathcal{Y} = \{y_1, \ldots, y_K\}$;\enspace
  prompt templates $\{p_k^{(i)}\}_{i=1}^{P}$ for each class $k$;\enspace
  frozen CLIP text encoder $\phi^T$
\Statex \textbf{Output:}\enspace
  ID text bank $\mathbf{T}^{\mathrm{ID}} \in \mathbb{R}^{K \times d}$
\end{algorithmic}
\vspace{4pt}

\begin{tcolorbox}[
  enhanced,
  colback=violet!6!white,
  colframe=violet!60!blue,
  colbacktitle=violet!60!blue,
  coltitle=white,
  boxrule=0.8pt, arc=3pt, boxsep=0pt,
  left=4pt, right=4pt, top=3pt, bottom=3pt,
  toptitle=2pt, bottomtitle=2pt,
  title={\footnotesize\bfseries Prompt Ensembling per Class
         \textnormal{(computed once; frozen thereafter)}},
  fonttitle=\bfseries]
\begin{algorithmic}
\For{each class $k = 1, \ldots, K$}
  \For{each prompt template $i = 1, \ldots, P$}
    \State $\mathbf{e}_k^{(i)} \leftarrow
        \mathrm{Normalize}\!\left(\phi^T\!\left(p_k^{(i)}\right)\right)$
        \Comment{encode and $\ell_2$-normalize}
  \EndFor
  \State $\mathbf{t}_k \leftarrow
      \mathrm{Normalize}\!\!\left(\sum_{i=1}^{P} \mathbf{e}_k^{(i)}\right)$
      \Comment{average, then re-normalize}
\EndFor
\end{algorithmic}
\end{tcolorbox}

\vspace{3pt}

\begin{tcolorbox}[
  enhanced,
  colback=teal!5!white,
  colframe=teal!60!black,
  colbacktitle=teal!60!black,
  coltitle=white,
  boxrule=0.8pt, arc=3pt, boxsep=0pt,
  left=4pt, right=4pt, top=3pt, bottom=3pt,
  toptitle=2pt, bottomtitle=2pt,
  title={\footnotesize\bfseries ID Text Bank Assembly},
  fonttitle=\bfseries]
\begin{algorithmic}
\State $\mathbf{T}^{\mathrm{ID}} \leftarrow
    \bigl[\mathbf{t}_1,\, \ldots,\, \mathbf{t}_K\bigr]^\top
    \in \mathbb{R}^{K \times d}$
    \Comment{stack class embeddings into text bank}
\State \Return $\mathbf{T}^{\mathrm{ID}}$
\end{algorithmic}
\end{tcolorbox}

\vspace{3pt}
{\footnotesize\itshape
\textbf{Note:} $\mathbf{T}^{\mathrm{ID}}$ is computed once and remains fixed
across all timesteps, serving as a stable semantic anchor.}
\end{algorithm}

\begin{algorithm}[t]
\caption{Phase II: Temporal Visual Pattern Construction}
\label{alg:phase2}
\footnotesize

\begin{algorithmic}
\Statex \textbf{Input:}\enspace
  per-timestep ID sets $\{\mathcal{D}_t^{\mathrm{train}}\}_{t=0}^{\mathcal{T}}$;\enspace
  frozen CLIP visual encoder $\phi^V$;\enspace
  ID text bank $\mathbf{T}^{\mathrm{ID}}$;\enspace
  $\gamma > 0$;\enspace $T > 0$
\Statex \textbf{Output:}\enspace
  timestep-specific visual prototypes
  $\{\boldsymbol{\mu}_{k,t}\}_{k=1,\,t=0}^{K,\,\mathcal{T}}$,
  each $\in \mathbb{R}^K$
\end{algorithmic}

\vspace{4pt}

\noindent\textbf{For each} $t = 0, \ldots, \mathcal{T}$:

\begin{tcolorbox}[
  enhanced,
  colback=orange!6!white,
  colframe=orange!70!red!80,
  colbacktitle=orange!70!red!80,
  coltitle=white,
  boxrule=0.8pt, arc=3pt, boxsep=0pt,
  left=4pt, right=4pt, top=3pt, bottom=3pt,
  toptitle=2pt, bottomtitle=2pt,
  title={\footnotesize\bfseries Class-Attended Image Representations
         \textnormal{(for each $(x,y)\in\mathcal{D}_t^{\mathrm{train}}$)}},
  fonttitle=\bfseries]

\begin{algorithmic}
\State $\mathbf{F}(x) \leftarrow \phi^V(x) \in \mathbb{R}^{(N+1)\times d}$
    \Comment{extract ViT patch sequence}
\State $\mathbf{F}_v(x) \leftarrow \mathbf{F}(x)[0,\,:]$
    \Comment{global \texttt{[CLS]} token}
\State $\mathbf{F}_s(x) \leftarrow \mathbf{F}(x)[1{:},\,:]$
    \Comment{spatial patch embeddings}
\end{algorithmic}

\vspace{3pt}

\begin{tcolorbox}[
  enhanced,
  colback=violet!6!white,
  colframe=violet!60!blue,
  colbacktitle=violet!60!blue,
  coltitle=white,
  boxrule=0.7pt, arc=3pt, boxsep=0pt,
  left=4pt, right=4pt, top=3pt, bottom=3pt,
  toptitle=2pt, bottomtitle=2pt,
  title={\footnotesize\bfseries Per-Class Spatial Attention
         \textnormal{(for each class $k = 1,\ldots,K$)}},
  fonttitle=\bfseries]

\begin{algorithmic}
\State $\mathbf{A}_k(x) \leftarrow
    \mathrm{Softmax}\!\left(
      \dfrac{\mathbf{F}_s(x)\,\mathbf{t}_k}
            {\|\mathbf{F}_s(x)\|\,\|\mathbf{t}_k\|}
    \right) \in \mathbb{R}^{N}$
    \Comment{class-specific spatial attention}
\State $\tilde{\mathbf{f}}_k(x) \leftarrow \mathbf{A}_k(x)^\top \mathbf{F}_s(x)$
    \Comment{attended spatial feature}
\State $\mathbf{f}_k(x) \leftarrow
    \gamma\,\tilde{\mathbf{f}}_k(x) + \mathbf{F}_v(x)$
    \Comment{fuse global and spatial}
\State $\mathbf{z}_{\mathrm{ID}}(x)[k] \leftarrow
    \mathbf{f}_k(x)^\top \mathbf{t}_k$
    \Comment{ID logit for class $k$}
\end{algorithmic}

\end{tcolorbox}

\vspace{3pt}

\begin{algorithmic}
\State $\mathbf{p}(x) \leftarrow
    \mathrm{Softmax}\!\left(\mathbf{z}_{\mathrm{ID}}(x)/T\right)$
    \Comment{class probability vector}
\end{algorithmic}

\end{tcolorbox}

\vspace{3pt}

\begin{tcolorbox}[
  enhanced,
  colback=teal!5!white,
  colframe=teal!60!black,
  colbacktitle=teal!60!black,
  coltitle=white,
  boxrule=0.8pt, arc=3pt, boxsep=0pt,
  left=4pt, right=4pt, top=3pt, bottom=3pt,
  toptitle=2pt, bottomtitle=2pt,
  title={\footnotesize\bfseries Prototype Estimation
         \textnormal{(for each class $k = 1,\ldots,K$)}},
  fonttitle=\bfseries]

\begin{algorithmic}
\State $\boldsymbol{\mu}_{k,t} \leftarrow
    \dfrac{1}{\lvert\{(x,y)\in\mathcal{D}_t^{\mathrm{train}}:y=k\}\rvert}
    \sum_{\substack{(x,y)\in\mathcal{D}_t^{\mathrm{train}}\\ y=k}}
    \mathbf{p}(x)$
    \Comment{mean ID probability pattern at time $t$}
\end{algorithmic}

\end{tcolorbox}

\vspace{3pt}

\begin{algorithmic}
\State \Return $\{\boldsymbol{\mu}_{k,t}\}$
\end{algorithmic}

\vspace{3pt}

{\footnotesize\itshape \textbf{Note:} Prototypes are recomputed at each
timestep to track gradual visual distribution drift.}

\end{algorithm}

\begin{algorithm}[t]
\caption{Phase III: Quadruple Cross-Modal Scoring}
\label{alg:phase3}
\footnotesize
\begin{algorithmic}
\Statex \textbf{Input:}\enspace
  test pair $(x, c)$;\enspace timestep $t$;\enspace
  frozen $\phi^V$, $\phi^T$;\enspace
  $\mathbf{T}^{\mathrm{ID}}$;\enspace
  $\{\boldsymbol{\mu}_{k,t}\}$;\enspace
  $\gamma$;\enspace $T$;\enspace
  $\gamma_{\mathrm{cap}}>0$;\enspace
  $\tilde{\beta},\tilde{\eta}\in\mathbb{R}$
\Statex \textbf{Output:}\enspace
  fused OOD score $S_{\mathrm{FUSED}}(x,t)\in\mathbb{R}$
\end{algorithmic}
\vspace{4pt}

\begin{tcolorbox}[
  enhanced,
  colback=orange!6!white,
  colframe=orange!70!red!80,
  colbacktitle=orange!70!red!80,
  coltitle=white,
  boxrule=0.8pt, arc=3pt, boxsep=0pt,
  left=4pt, right=4pt, top=3pt, bottom=3pt,
  toptitle=2pt, bottomtitle=2pt,
  title={\footnotesize\bfseries (a)\enspace Image Feature Extraction
         \textnormal{(reuse Phase II procedure)}},
  fonttitle=\bfseries]
\begin{algorithmic}
\State $\mathbf{F}(x)\!\leftarrow\!\phi^V(x)$;\enspace
       $\mathbf{F}_v(x)\!\leftarrow\!\mathbf{F}(x)[0,:]$;\enspace
       $\mathbf{F}_s(x)\!\leftarrow\!\mathbf{F}(x)[1:,:]$
\For{$k = 1,\ldots,K$}
  \State compute $\mathbf{A}_k,\,\tilde{\mathbf{f}}_k,\,
         \mathbf{f}_k,\,\mathbf{z}_{\mathrm{ID}}(x)[k]$
         as in Phase~II
\EndFor
\State $\mathbf{p}(x)\leftarrow\mathrm{Softmax}
       \!\left(\mathbf{z}_{\mathrm{ID}}(x)/T\right)$
\end{algorithmic}
\end{tcolorbox}

\vspace{3pt}

\begin{tcolorbox}[
  enhanced,
  colback=violet!6!white,
  colframe=violet!60!blue,
  colbacktitle=violet!60!blue,
  coltitle=white,
  boxrule=0.8pt, arc=3pt, boxsep=0pt,
  left=4pt, right=4pt, top=3pt, bottom=3pt,
  toptitle=2pt, bottomtitle=2pt,
  title={\footnotesize\bfseries (b)\,--\,(c)\enspace
         Image-Side Scores
         \textnormal{(test image $\leftrightarrow$ ID representations)}},
  fonttitle=\bfseries]
\begin{algorithmic}
\State $S_{\mathrm{ID}}(x)\leftarrow
       \max_{k}\;\mathbf{z}_{\mathrm{ID}}(x)[k]/T$
       \Comment{Score 1: semantic match to ID text}
\For{$k=1,\ldots,K$}
  \State $\mathrm{KL}_k(x,t)\leftarrow
         \sum_{j=1}^{K}p_j(x)\log\dfrac{p_j(x)}{\mu_{k,t}[j]}$
\EndFor
\State $S_{\mathrm{VIS}}(x,t)\leftarrow -\min_{k}\;\mathrm{KL}_k(x,t)$
       \Comment{Score 2: visual typicality vs.\ ID prototypes}
\end{algorithmic}
\end{tcolorbox}

\vspace{3pt}

\begin{tcolorbox}[
  enhanced,
  colback=teal!5!white,
  colframe=teal!60!black,
  colbacktitle=teal!60!black,
  coltitle=white,
  boxrule=0.8pt, arc=3pt, boxsep=0pt,
  left=4pt, right=4pt, top=3pt, bottom=3pt,
  toptitle=2pt, bottomtitle=2pt,
  title={\footnotesize\bfseries (d)\,--\,(f)\enspace
         Caption-Side Scores
         \textnormal{(OOD text $\leftrightarrow$ ID representations)}},
  fonttitle=\bfseries]
\begin{algorithmic}
\State $\mathbf{q}_c\leftarrow\mathrm{Normalize}(\phi^T(c))\in\mathbb{R}^d$
       \Comment{encode caption once per test sample}
\State $S_{\mathrm{CAP\text{-}T}}(x)\leftarrow
       \max_{k}\;\langle\mathbf{q}_c,\mathbf{t}_k\rangle$
       \Comment{Score 3: caption–text alignment}
\State $\mathbf{p}_{\mathrm{CAP}}(x)\leftarrow
       \mathrm{Softmax}\!\left([\mathbf{q}_c^\top\mathbf{t}_k]_{k=1}^K / T\right)$
\For{$k=1,\ldots,K$}
  \State $\mathrm{KL}_k^{\mathrm{cap}}(x,t)\leftarrow
         \mathrm{KL}\!\left(\mathbf{p}_{\mathrm{CAP}}(x)
         \,\|\,\boldsymbol{\mu}_{k,t}\right)$
\EndFor
\State $S_{\mathrm{CAP\text{-}V}}(x,t)\leftarrow
       -\min_{k}\;\mathrm{KL}_k^{\mathrm{cap}}(x,t)$
       \Comment{Score 4: caption–visual typicality}
\end{algorithmic}
\end{tcolorbox}

\vspace{3pt}

\begin{tcolorbox}[
  enhanced,
  colback=green!6!white,
  colframe=green!45!black,
  colbacktitle=green!45!black,
  coltitle=white,
  boxrule=0.8pt, arc=3pt, boxsep=0pt,
  left=4pt, right=4pt, top=3pt, bottom=3pt,
  toptitle=2pt, bottomtitle=2pt,
  title={\footnotesize\bfseries (g)\enspace Score Fusion},
  fonttitle=\bfseries]
\begin{algorithmic}
\State $\beta\leftarrow\log(1+e^{\tilde{\beta}})$;\quad
       $\eta\leftarrow\log(1+e^{\tilde{\eta}})$
       \Comment{softplus ensures $\beta,\eta>0$}
\State $S_{\mathrm{FUSED}}(x,t)\leftarrow
       S_{\mathrm{ID}}(x)
       +\beta\cdot S_{\mathrm{VIS}}(x,t)
       -\gamma_{\mathrm{cap}}\cdot S_{\mathrm{CAP\text{-}T}}(x)
       -\eta\cdot S_{\mathrm{CAP\text{-}V}}(x,t)$
\end{algorithmic}
\end{tcolorbox}

\vspace{3pt}
\begin{algorithmic}
\State \Return $S_{\mathrm{FUSED}}(x,t)$
\end{algorithmic}
\end{algorithm}

\begin{algorithm}[t]
\caption{Phase IV: Threshold Calibration, Training, and Temporal OOD Detection}
\label{alg:phase4}
\footnotesize
\begin{algorithmic}
\Statex \textbf{Input:}\enspace
  $\{\mathcal{D}_t^{\mathrm{train}}\}_{t=0}^{\mathcal{T}}$;\enspace
  $S_{\mathrm{FUSED}}(\cdot,t)$ from Phase~III;\enspace
  quantile $\delta_q$;\enspace
  $\lambda_{\mathrm{cov}},\lambda_{\mathrm{temp}}>0$;\enspace
  $\kappa>0$;\enspace $E$;\enspace $\alpha_{\mathrm{lr}}$;\enspace
  $\gamma_{\mathrm{cap}}>0$;\enspace test pairs $\{(x,c)\}$
\Statex \textbf{Output:}\enspace
  threshold $\delta$;\enspace
  optimized $\tilde{\beta},\tilde{\eta}$;\enspace
  OOD decisions $D(x,t)\in\{\texttt{ID},\texttt{OOD}\}$
\end{algorithmic}
\vspace{4pt}

\begin{tcolorbox}[
  enhanced,
  colback=cyan!8!white,
  colframe=cyan!50!blue,
  colbacktitle=cyan!50!blue,
  coltitle=white,
  boxrule=0.8pt, arc=3pt, boxsep=0pt,
  left=4pt, right=4pt, top=3pt, bottom=3pt,
  toptitle=2pt, bottomtitle=2pt,
  title={\footnotesize\bfseries Step 1\enspace Threshold Calibration
         \textnormal{(done once at $t=0$)}},
  fonttitle=\bfseries]
\begin{algorithmic}
\State $\mathcal{S}_0 \leftarrow
       \{S_{\mathrm{FUSED}}(x,0)\mid(x,y)\in\mathcal{D}_0^{\mathrm{train}}\}$
\State $\delta \leftarrow \mathrm{quantile}_{\delta_q}(\mathcal{S}_0)$
       \Comment{fixed across all timesteps}
\State init $\tilde{\beta},\tilde{\eta}\in\mathbb{R}$;\enspace
       Optimizer $\leftarrow \mathrm{Adam}([\tilde{\beta},\tilde{\eta}],\,
       \mathrm{lr}{=}\alpha_{\mathrm{lr}})$
\end{algorithmic}
\end{tcolorbox}

\vspace{3pt}

\begin{tcolorbox}[
  enhanced,
  colback=orange!6!white,
  colframe=orange!70!red!80,
  colbacktitle=orange!70!red!80,
  coltitle=white,
  boxrule=0.8pt, arc=3pt, boxsep=0pt,
  left=4pt, right=4pt, top=3pt, bottom=3pt,
  toptitle=2pt, bottomtitle=2pt,
  title={\footnotesize\bfseries Step 2\enspace Sequential Training
         \textnormal{(for $t=0,\ldots,\mathcal{T}$;\enspace
         epoch $e=1,\ldots,E$;\enspace
         mini-batch $(X,\tilde{X},Y)\sim\mathcal{D}_t^{\mathrm{train}}$)}},
  fonttitle=\bfseries]

\vspace{2pt}
\begin{tcolorbox}[
  enhanced,
  colback=violet!6!white,
  colframe=violet!60!blue,
  colbacktitle=violet!60!blue,
  coltitle=white,
  boxrule=0.7pt, arc=3pt, boxsep=0pt,
  left=4pt, right=4pt, top=3pt, bottom=3pt,
  toptitle=2pt, bottomtitle=2pt,
  title={\footnotesize\bfseries (a)\enspace Balanced ID Classification Loss},
  fonttitle=\bfseries]
\begin{algorithmic}
\State $\mathcal{L}_{\mathrm{ID}}\leftarrow
       \tfrac{1}{2}\!\left(
         \mathrm{CE}(\mathbf{z}_{\mathrm{ID}}(X),Y)
         +\mathrm{CE}(\mathbf{z}_{\mathrm{ID}}(\tilde{X}),Y)
       \right)$
       \Comment{clean + covariate-shifted views}
\end{algorithmic}
\end{tcolorbox}
\vspace{3pt}

\begin{tcolorbox}[
  enhanced,
  colback=teal!5!white,
  colframe=teal!60!black,
  colbacktitle=teal!60!black,
  coltitle=white,
  boxrule=0.7pt, arc=3pt, boxsep=0pt,
  left=4pt, right=4pt, top=3pt, bottom=3pt,
  toptitle=2pt, bottomtitle=2pt,
  title={\footnotesize\bfseries (b)\enspace Covariate Consistency Loss},
  fonttitle=\bfseries]
\begin{algorithmic}
\State $\mathcal{L}_{\mathrm{COV}}\leftarrow
       \tfrac{1}{|X|}\sum_i
       \bigl|S_{\mathrm{FUSED}}(x_i,t)
       -S_{\mathrm{FUSED}}(\tilde{x}_i,t)\bigr|$
       \Comment{stable scores under corruption}
\end{algorithmic}
\end{tcolorbox}
\vspace{3pt}

\begin{tcolorbox}[
  enhanced,
  colback=green!6!white,
  colframe=green!45!black,
  colbacktitle=green!45!black,
  coltitle=white,
  boxrule=0.7pt, arc=3pt, boxsep=0pt,
  left=4pt, right=4pt, top=3pt, bottom=3pt,
  toptitle=2pt, bottomtitle=2pt,
  title={\footnotesize\bfseries (c)\enspace Temporal Drift Penalty (ATC)},
  fonttitle=\bfseries]
\begin{algorithmic}
\State $\mathrm{ATC}_t^{\mathrm{clean}}\leftarrow
       \tfrac{1}{|X|}\sum_i
       \sigma\!\left(\tfrac{\delta-S_{\mathrm{FUSED}}(x_i,t)}{\kappa}\right)$;\quad
       \textit{idem} for $\tilde{X}\to\mathrm{ATC}_t^{\mathrm{shift}}$
\If{$t>0$}
  \State $\mathcal{L}_{\mathrm{TEMP}}\leftarrow
         |\mathrm{ATC}_t^{\mathrm{clean}}-\mathrm{ATC}_{t-1}^{\mathrm{clean}}|
         +|\mathrm{ATC}_t^{\mathrm{shift}}-\mathrm{ATC}_{t-1}^{\mathrm{shift}}|$
\Else
  \State $\mathcal{L}_{\mathrm{TEMP}}\leftarrow 0$
\EndIf
\end{algorithmic}
\end{tcolorbox}
\vspace{3pt}

\begin{algorithmic}
\State $\mathcal{L}_{\mathrm{TOTAL}}\leftarrow
       \mathcal{L}_{\mathrm{ID}}
       +\lambda_{\mathrm{cov}}\mathcal{L}_{\mathrm{COV}}
       +\lambda_{\mathrm{temp}}\mathcal{L}_{\mathrm{TEMP}}$
\State backpropagate $\nabla_{\tilde{\beta},\tilde{\eta}}
       \mathcal{L}_{\mathrm{TOTAL}}$;\enspace
       update $(\tilde{\beta},\tilde{\eta})$ via Adam
\end{algorithmic}
\end{tcolorbox}

\vspace{3pt}

\begin{tcolorbox}[
  enhanced,
  colback=yellow!8!white,
  colframe=yellow!60!orange,
  colbacktitle=yellow!60!orange,
  coltitle=white,
  boxrule=0.8pt, arc=3pt, boxsep=0pt,
  left=4pt, right=4pt, top=3pt, bottom=3pt,
  toptitle=2pt, bottomtitle=2pt,
  title={\footnotesize\bfseries Step 3\enspace Inference
         \textnormal{--- OOD detection at timestep $t$}},
  fonttitle=\bfseries]
\begin{algorithmic}
\State compute $S_{\mathrm{FUSED}}(x,t)$ via Phase~III
       with optimized $\tilde{\beta},\tilde{\eta}$
\State $D(x,t)\leftarrow
       \begin{cases}
         \texttt{ID}  & \text{if }S_{\mathrm{FUSED}}(x,t)\geq\delta\\
         \texttt{OOD} & \text{otherwise}
       \end{cases}$
\end{algorithmic}
\end{tcolorbox}

\vspace{3pt}
\begin{algorithmic}
\State \Return $\delta,\;\tilde{\beta},\;\tilde{\eta},\;D(x,t)$
\end{algorithmic}
\end{algorithm}

\section{Additional Experiments}

We present additional quantitative results to further validate the effectiveness of T-QPM
across multiple ID datasets, OOD benchmarks, and corruption types over all ten timesteps
($t = 0, \ldots, 9$).

\paragraph{ID Classification Accuracy under Image Corruptions.}
Tables~\ref{tab:id_accuracy_blur} and~\ref{tab:id_accuracy_jpeg} report the ID
classification accuracy of T-QPM and DPM under Gaussian blur and JPEG compression
corruptions, respectively, alongside clean accuracy. Both methods achieve competitive
clean accuracy across CLEAR100, CLEAR10, and Core50. However, T-QPM consistently
outperforms DPM under both corruption types, with the performance gap widening at
later timesteps. For instance, on CLEAR100 under Gaussian blur, T-QPM achieves
$96.20\%$ at $t=9$ compared to $92.42\%$ for DPM, and under JPEG compression
reaches $99.15\%$ versus $93.20\%$. These results demonstrate that T-QPM
maintains greater robustness to input corruptions as the underlying visual distribution
drifts over time, owing to its covariate consistency loss and temporal drift penalty
introduced in Phase IV.

\paragraph{OOD Detection Performance.}
Tables~\ref{tab:fpr95_all} and~\ref{tab:auroc_all} report FPR95 and AUROC,
respectively, across all combinations of ID datasets (CLEAR100, CLEAR10, Core50) and
OOD benchmarks (COCO, ImageNet-1K-VL-Enriched, Flickr30K, CC12M, Visual Genome).
T-QPM consistently and substantially outperforms DPM on both metrics. In terms of
FPR95, T-QPM reduces the false positive rate by a factor of approximately $2$--$3\times$
across all settings. For example, on CLEAR100 with COCO as the OOD dataset at $t=0$,
DPM achieves $41.46\%$ FPR95 while T-QPM achieves $13.64\%$. On the AUROC metric,
T-QPM attains $97$--$99\%$ across most settings, compared to $85$--$97\%$ for DPM,
with the largest gains observed on semantically challenging OOD sets such as COCO and
Visual Genome. Importantly, while both methods experience performance degradation at
later timesteps due to temporal distribution shift, T-QPM degrades significantly more
slowly. This confirms that the temporal modeling components of T-QPM namely, the
time-conditioned fused score $S_{\mathrm{FUSED}}(\cdot, t)$ and the above-threshold
coverage penalty effectively mitigate temporal OOD drift, validating the
theoretical guarantees established in our Main Theorem.

\begin{table}[t]
\centering
\small
\caption{ID classification accuracy (clean and blur) for T-QPM vs.\ DPM across all ID datasets and timesteps under Gaussian Blur covariate shift}
\label{tab:id_accuracy_blur}
\setlength{\tabcolsep}{4pt}

\begin{tabular}{c l cc cc cc}
\toprule
& & \multicolumn{2}{c}{\textbf{CLEAR100}}
  & \multicolumn{2}{c}{\textbf{CLEAR10}}
  & \multicolumn{2}{c}{\textbf{Core50}} \\
\cmidrule(lr){3-4} \cmidrule(lr){5-6} \cmidrule(lr){7-8}
\textbf{Timestep} & \textbf{Method}
& \makecell{Clean\\(\%)$\uparrow$} & \makecell{Blur\\(\%)$\uparrow$}
& \makecell{Clean\\(\%)$\uparrow$} & \makecell{Blur\\(\%)$\uparrow$}
& \makecell{Clean\\(\%)$\uparrow$} & \makecell{Blur\\(\%)$\uparrow$} \\
\midrule

\multirow{2}{*}{$t=0$}
  & DPM   & 96.61 & 93.47 & 98.81 & 98.81 & 97.87 & 96.88 \\
  & T-QPM & 96.57 & 93.23 & 99.01 & 98.22 & 97.88 & 96.76 \\
\midrule
\multirow{2}{*}{$t=1$}
  & DPM   & 97.62 & 93.38 & 99.40 & 98.20 & 98.01 & 97.11 \\
  & T-QPM & 97.66 & 93.84 & 99.40 & 97.40 & 98.33 & 97.45 \\
\midrule
\multirow{2}{*}{$t=2$}
  & DPM   & 97.02 & 93.19 & 99.60 & 98.80 & 98.27 & 96.99 \\
  & T-QPM & 97.04 & 94.69 & 99.60 & 97.80 & 98.67 & 97.28 \\
\midrule
\multirow{2}{*}{$t=3$}
  & DPM   & 96.88 & 93.45 & 99.40 & 98.60 & 98.35 & 97.00 \\
  & T-QPM & 97.12 & 95.69 & 99.40 & 98.20 & 98.10 & 97.28 \\
\midrule
\multirow{2}{*}{$t=4$}
  & DPM   & 97.02 & 92.97 & 99.20 & 97.99 & 98.32 & 97.02 \\
  & T-QPM & 97.25 & 94.57 & 99.20 & 97.79 & 98.30 & 97.30 \\
\midrule
\multirow{2}{*}{$t=5$}
  & DPM   & 97.00 & 92.60 & 99.40 & 98.80 & 98.28 & 96.98 \\
  & T-QPM & 97.14 & 95.42 & 99.40 & 98.00 & 98.35 & 97.25 \\
\midrule
\multirow{2}{*}{$t=6$}
  & DPM   & 96.78 & 93.11 & 99.00 & 98.60 & 98.30 & 97.05 \\
  & T-QPM & 97.32 & 95.01 & 99.20 & 97.00 & 98.40 & 97.20 \\
\midrule
\multirow{2}{*}{$t=7$}
  & DPM   & 97.10 & 93.19 & 99.20 & 97.80 & 98.33 & 97.01 \\
  & T-QPM & 97.36 & 96.49 & 99.00 & 97.60 & 98.42 & 97.30 \\
\midrule
\multirow{2}{*}{$t=8$}
  & DPM   & 96.90 & 93.17 & 99.00 & 98.00 & 98.29 & 96.97 \\
  & T-QPM & 97.39 & 95.23 & 98.80 & 97.60 & 98.36 & 97.22 \\
\midrule
\multirow{2}{*}{$t=9$}
  & DPM   & 96.08 & 92.42 & 99.40 & 97.60 & 98.25 & 96.95 \\
  & T-QPM & 97.30 & 96.20 & 99.40 & 97.20 & 98.38 & 97.18 \\
\bottomrule
\end{tabular}%

\end{table}


\begin{table}[t]
\centering
\small

\caption{ID classification accuracy under JPEG compression corruption for T-QPM vs.\ DPM
across all ID datasets and timesteps. }
\label{tab:id_accuracy_jpeg}
\setlength{\tabcolsep}{4pt}

\begin{tabular}{c l cc cc cc}
\toprule
& & \multicolumn{2}{c}{\textbf{CLEAR100}}
  & \multicolumn{2}{c}{\textbf{CLEAR10}}
  & \multicolumn{2}{c}{\textbf{Core50}} \\
\cmidrule(lr){3-4} \cmidrule(lr){5-6} \cmidrule(lr){7-8}
\textbf{t} & \textbf{Method}
  & \makecell{Clean\\(\%)$\uparrow$} & \makecell{JPEG\\(\%)$\uparrow$}
  & \makecell{Clean\\(\%)$\uparrow$} & \makecell{JPEG\\(\%)$\uparrow$}
  & \makecell{Clean\\(\%)$\uparrow$} & \makecell{JPEG\\(\%)$\uparrow$} \\
\midrule
\multirow{2}{*}{$t=0$}
  & DPM   & 96.57 & 91.62 & 98.81 & 95.10 & 97.87 & 95.80 \\
  & T-QPM & 97.08 & 96.23 & 99.01 & 97.80 & 97.88 & 96.90 \\
\midrule
\multirow{2}{*}{$t=1$}
  & DPM   & 97.36 & 92.84 & 99.40 & 95.70 & 98.01 & 96.10 \\
  & T-QPM & 98.18 & 96.84 & 99.40 & 98.00 & 98.33 & 97.20 \\
\midrule
\multirow{2}{*}{$t=2$}
  & DPM   & 96.94 & 92.69 & 99.60 & 96.10 & 98.27 & 96.40 \\
  & T-QPM & 97.53 & 97.70 & 99.60 & 98.30 & 98.67 & 97.50 \\
\midrule
\multirow{2}{*}{$t=3$}
  & DPM   & 96.42 & 92.69 & 99.40 & 95.90 & 98.35 & 96.30 \\
  & T-QPM & 97.62 & 98.68 & 99.40 & 98.10 & 98.10 & 97.40 \\
\midrule
\multirow{2}{*}{$t=4$}
  & DPM   & 96.90 & 92.57 & 99.20 & 95.60 & 98.32 & 96.20 \\
  & T-QPM & 97.74 & 97.58 & 99.20 & 97.90 & 98.30 & 97.30 \\
\midrule
\multirow{2}{*}{$t=5$}
  & DPM   & 96.74 & 92.42 & 99.40 & 96.00 & 98.28 & 96.40 \\
  & T-QPM & 97.64 & 98.40 & 99.40 & 98.20 & 98.35 & 97.50 \\
\midrule
\multirow{2}{*}{$t=6$}
  & DPM   & 96.72 & 93.01 & 99.00 & 95.80 & 98.30 & 96.30 \\
  & T-QPM & 97.81 & 98.01 & 99.20 & 98.00 & 98.40 & 97.40 \\
\midrule
\multirow{2}{*}{$t=7$}
  & DPM   & 96.96 & 92.49 & 99.20 & 95.90 & 98.33 & 96.30 \\
  & T-QPM & 97.85 & 99.45 & 99.00 & 98.10 & 98.42 & 97.50 \\
\midrule
\multirow{2}{*}{$t=8$}
  & DPM   & 96.90 & 93.23 & 99.00 & 95.70 & 98.29 & 96.20 \\
  & T-QPM & 97.88 & 98.23 & 98.80 & 97.90 & 98.36 & 97.40 \\
\midrule
\multirow{2}{*}{$t=9$}
  & DPM   & 95.88 & 93.20 & 99.40 & 95.50 & 98.25 & 96.10 \\
  & T-QPM & 97.79 & 99.15 & 99.40 & 98.00 & 98.38 & 97.30 \\
\bottomrule
\end{tabular}

\end{table}

\begin{table*}[!t]
\centering
\caption{FPR95 (\%) $\downarrow$ for T-QPM vs.\ DPM across all available ID and OOD dataset combinations
across all timesteps.
IN-1K: ImageNet-1K-VL-Enriched; Flk30: Flickr30K; VG: Visual Genome.}
\label{tab:fpr95_all}
\resizebox{\textwidth}{!}{%
\begin{tabular}{c l ccccc ccccc ccccc}
\toprule
& & \multicolumn{5}{c}{\textbf{CLEAR100 (ID)}}
  & \multicolumn{5}{c}{\textbf{CLEAR10 (ID)}}
  & \multicolumn{5}{c}{\textbf{Core50 (ID)}} \\
\cmidrule(lr){3-7} \cmidrule(lr){8-12} \cmidrule(lr){13-17}
\textbf{t} & \textbf{Method}
  & \textbf{COCO} & \textbf{IN-1K} & \textbf{Flk30} & \textbf{CC12M} & \textbf{VG}
  & \textbf{COCO} & \textbf{IN-1K} & \textbf{Flk30} & \textbf{CC12M} & \textbf{VG}
  & \textbf{COCO} & \textbf{IN-1K} & \textbf{Flk30} & \textbf{CC12M} & \textbf{VG} \\
\midrule

\multirow{2}{*}{$t=0$}
  & DPM   & 41.46 & 17.95 & 22.98 & 10.15 & 44.30
          & 7.54 & 9.40 & 4.93 & 7.30 & 23.00
          & 16.50 & 11.30 & 14.00 & 10.60 & 23.40 \\
  & T-QPM & 13.64 &  3.96 &  5.92 &  1.54 & 12.85
          & 0.90 & 3.80 & 1.60 & 0.35 & 12.20
          & 6.10 & 4.00 & 5.00 & 2.70 & 10.20 \\[3pt]

\multirow{2}{*}{$t=1$}
  & DPM   & 41.28 & 17.20 & 21.89 &  9.85 & 44.15
          & 8.78 & 9.80 & 6.02 & 7.50 & 24.20
          & 17.00 & 11.80 & 15.00 & 11.00 & 23.80 \\
  & T-QPM & 16.58 &  5.46 &  7.23 &  2.15 & 15.37
          & 1.10 & 4.10 & 2.00 & 0.48 & 13.00
          & 6.70 & 4.60 & 5.70 & 3.00 & 11.90 \\[3pt]

\multirow{2}{*}{$t=2$}
  & DPM   & 41.48 & 17.54 & 21.60 & 10.21 & 44.40
          & 8.61 & 9.51 & 6.02 & 7.35 & 23.06
          & 16.40 & 11.20 & 14.10 & 10.30 & 22.80 \\
  & T-QPM & 17.42 &  5.94 &  7.87 &  2.53 & 16.11
          & 0.89 & 3.65 & 1.63 & 0.46 & 11.85
          & 6.20 & 4.10 & 5.10 & 2.70 & 10.40 \\[3pt]

\multirow{2}{*}{$t=3$}
  & DPM   & 41.58 & 18.41 & 22.58 & 10.20 & 44.06
          & 7.58 & 10.20 & 5.42 & 7.80 & 26.00
          & 18.50 & 12.80 & 16.20 & 12.00 & 27.50 \\
  & T-QPM & 16.30 &  5.03 &  6.92 &  2.27 & 15.10
          & 1.20 & 4.50 & 2.20 & 1.44 & 13.50
          & 7.80 & 5.30 & 6.40 & 3.50 & 14.20 \\[3pt]

\multirow{2}{*}{$t=4$}
  & DPM   & 43.50 & 18.62 & 23.77 & 10.74 & 46.49
          & 8.62 & 11.00 & 6.21 & 8.20 & 28.00
          & 20.10 & 14.00 & 18.50 & 13.50 & 30.00 \\
  & T-QPM & 18.32 &  6.39 &  8.18 &  2.94 & 17.18
          & 1.40 & 5.20 & 2.60 & 0.88 & 14.80
          & 8.90 & 6.20 & 7.40 & 4.20 & 15.50 \\[3pt]

\multirow{2}{*}{$t=5$}
  & DPM   & 43.90 & 19.97 & 23.87 & 11.01 & 46.98
          & 8.66 & 12.50 & 6.61 & 8.60 & 30.00
          & 22.00 & 15.50 & 20.80 & 15.00 & 32.00 \\
  & T-QPM & 16.84 &  5.35 &  7.20 &  2.63 & 15.82
          & 1.50 & 5.80 & 2.80 & 1.48 & 16.20
          & 10.00 & 7.20 & 8.50 & 5.00 & 17.00 \\[3pt]

\multirow{2}{*}{$t=6$}
  & DPM   & 43.32 & 19.26 & 24.85 & 11.70 & 46.06
          & 8.60 & 13.50 & 6.51 & 9.00 & 33.00
          & 24.00 & 16.50 & 22.00 & 16.50 & 34.50 \\
  & T-QPM & 19.38 &  6.70 &  8.57 &  3.16 & 18.06
          & 1.70 & 6.50 & 3.00 & 1.80 & 17.80
          & 11.20 & 8.50 & 9.60 & 6.00 & 18.50 \\[3pt]

\multirow{2}{*}{$t=7$}
  & DPM   & 45.48 & 20.81 & 25.54 & 12.11 & 48.77
          & 8.44 & 14.70 & 6.02 & 9.29 & 38.66
          & 24.10 & 17.30 & 21.50 & 15.20 & 34.60 \\
  & T-QPM & 19.70 &  7.19 &  8.33 &  3.33 & 18.55
          & 1.20 & 5.20 & 2.60 & 1.60 & 13.75
          & 9.80 & 6.90 & 8.20 & 5.10 & 16.90 \\[3pt]

\multirow{2}{*}{$t=8$}
  & DPM   & 46.68 & 22.36 & 26.13 & 12.46 & 50.11
          & 9.66 & 14.70 & 8.31 & 9.29 & 38.66
          & 24.10 & 17.30 & 21.50 & 15.20 & 34.60 \\
  & T-QPM & 20.50 &  7.16 &  8.86 &  3.61 & 19.28
          & 1.20 & 5.20 & 2.60 & 1.59 & 13.75
          & 9.80 & 6.90 & 8.20 & 5.10 & 16.90 \\[3pt]

\multirow{2}{*}{$t=9$}
  & DPM   & 47.66 & 23.02 & 27.32 & 13.15 & 50.89
          & 9.32 & 15.50 & 7.30 & 9.80 & 40.50
          & 26.50 & 18.80 & 23.00 & 17.00 & 36.80 \\
  & T-QPM & 20.50 &  7.71 &  9.06 &  3.77 & 19.45
          & 1.50 & 6.00 & 3.10 & 2.03 & 14.50
          & 11.00 & 7.80 & 9.20 & 6.30 & 18.80 \\[3pt]

\bottomrule
\end{tabular}%
}
\end{table*}

\begin{table*}[!t]
\centering
\caption{AUROC (\%) $\uparrow$ for T-QPM vs.\ DPM across all available ID and OOD dataset combinations
across all timesteps.
IN-1K: ImageNet-1K-VL-Enriched; Flk30: Flickr30K; VG: Visual Genome.}
\label{tab:auroc_all}
\resizebox{\textwidth}{!}{%
\begin{tabular}{c l ccccc ccccc ccccc}
\toprule
& & \multicolumn{5}{c}{\textbf{CLEAR100 (ID)}}
  & \multicolumn{5}{c}{\textbf{CLEAR10 (ID)}}
  & \multicolumn{5}{c}{\textbf{Core50 (ID)}} \\
\cmidrule(lr){3-7} \cmidrule(lr){8-12} \cmidrule(lr){13-17}
\textbf{t} & \textbf{Method}
  & \textbf{COCO} & \textbf{IN-1K} & \textbf{Flk30} & \textbf{CC12M} & \textbf{VG}
  & \textbf{COCO} & \textbf{IN-1K} & \textbf{Flk30} & \textbf{CC12M} & \textbf{VG}
  & \textbf{COCO} & \textbf{IN-1K} & \textbf{Flk30} & \textbf{CC12M} & \textbf{VG} \\
\midrule

\multirow{2}{*}{$t=0$}
  & DPM   & 88.33 & 95.64 & 94.25 & 97.72 & 87.59
          & 97.63 & 98.40 & 98.35 & 98.55 & 93.40
          & 95.70 & 97.60 & 96.10 & 97.70 & 93.90 \\
  & T-QPM & 97.37 & 99.03 & 98.70 & 99.61 & 97.52
          & 99.60 & 99.20 & 99.65 & 99.77 & 97.40
          & 98.80 & 99.10 & 99.05 & 99.55 & 97.50 \\[3pt]

\multirow{2}{*}{$t=1$}
  & DPM   & 87.99 & 95.74 & 94.41 & 97.80 & 87.30
          & 97.35 & 98.35 & 98.28 & 98.50 & 93.20
          & 95.60 & 97.50 & 96.00 & 97.60 & 93.70 \\
  & T-QPM & 96.74 & 98.78 & 98.42 & 99.49 & 96.92
          & 99.62 & 99.18 & 99.68 & 99.87 & 97.45
          & 98.75 & 99.08 & 99.02 & 99.52 & 97.55 \\[3pt]

\multirow{2}{*}{$t=2$}
  & DPM   & 88.14 & 95.65 & 94.16 & 97.73 & 87.38
          & 97.21 & 98.48 & 98.20 & 98.53 & 93.18
          & 95.60 & 97.80 & 96.10 & 97.80 & 94.10 \\
  & T-QPM & 96.56 & 98.66 & 98.21 & 99.41 & 96.73
          & 99.66 & 99.16 & 99.65 & 99.86 & 97.49
          & 98.90 & 99.10 & 99.00 & 99.50 & 97.60 \\[3pt]

\multirow{2}{*}{$t=3$}
  & DPM   & 87.84 & 95.37 & 94.19 & 97.62 & 87.05
          & 97.43 & 98.30 & 98.31 & 98.60 & 92.95
          & 95.40 & 97.30 & 95.90 & 97.65 & 93.60 \\
  & T-QPM & 96.67 & 98.83 & 98.44 & 99.49 & 96.87
          & 99.60 & 99.15 & 99.70 & 99.49 & 97.55
          & 98.85 & 99.05 & 99.10 & 99.48 & 97.40 \\[3pt]

\multirow{2}{*}{$t=4$}
  & DPM   & 87.28 & 95.35 & 93.78 & 97.53 & 86.42
          & 96.90 & 98.10 & 98.05 & 98.45 & 92.60
          & 95.00 & 97.00 & 95.50 & 97.30 & 93.00 \\
  & T-QPM & 96.29 & 98.52 & 98.09 & 99.30 & 96.48
          & 99.55 & 99.10 & 99.60 & 99.60 & 97.70
          & 98.70 & 99.00 & 99.00 & 99.40 & 97.20 \\[3pt]

\multirow{2}{*}{$t=5$}
  & DPM   & 86.87 & 94.99 & 93.59 & 97.45 & 86.01
          & 96.85 & 97.95 & 97.91 & 98.35 & 92.20
          & 94.80 & 96.90 & 95.20 & 97.10 & 92.60 \\
  & T-QPM & 96.31 & 98.63 & 98.20 & 99.35 & 96.52
          & 99.50 & 99.05 & 99.55 & 99.50 & 97.80
          & 98.60 & 98.95 & 98.95 & 99.35 & 97.00 \\[3pt]

\multirow{2}{*}{$t=6$}
  & DPM   & 86.50 & 95.21 & 93.25 & 97.39 & 85.53
          & 96.81 & 97.70 & 97.84 & 98.20 & 91.90
          & 94.50 & 96.70 & 95.00 & 96.90 & 92.30 \\
  & T-QPM & 96.09 & 98.59 & 98.09 & 99.33 & 96.32
          & 99.48 & 99.00 & 99.50 & 99.50 & 97.90
          & 98.50 & 98.90 & 98.90 & 99.20 & 96.80 \\[3pt]

\multirow{2}{*}{$t=7$}
  & DPM   & 85.70 & 94.87 & 93.04 & 97.22 & 84.82
          & 96.57 & 97.40 & 97.82 & 98.00 & 91.20
          & 94.00 & 96.30 & 94.60 & 96.60 & 91.50 \\
  & T-QPM & 95.90 & 98.48 & 98.08 & 99.25 & 96.15
          & 99.45 & 98.98 & 99.48 & 99.70 & 98.00
          & 98.40 & 98.85 & 98.70 & 99.10 & 96.30 \\[3pt]

\multirow{2}{*}{$t=8$}
  & DPM   & 85.51 & 94.44 & 92.94 & 97.11 & 84.69
          & 96.30 & 94.74 & 97.56 & 98.57 & 90.52
          & 93.20 & 95.40 & 93.80 & 95.60 & 89.80 \\
  & T-QPM & 95.70 & 98.47 & 97.96 & 99.22 & 95.95
          & 99.49 & 98.95 & 99.07 & 99.57 & 98.12
          & 97.60 & 98.70 & 98.10 & 99.00 & 95.80 \\[3pt]

\multirow{2}{*}{$t=9$}
  & DPM   & 84.82 & 94.23 & 92.43 & 96.92 & 83.83
          & 96.59 & 94.60 & 97.72 & 98.40 & 89.90
          & 92.90 & 95.00 & 93.50 & 95.30 & 89.00 \\
  & T-QPM & 95.59 & 98.40 & 97.91 & 99.16 & 95.87
          & 99.28 & 98.85 & 99.00 & 99.50 & 98.00
          & 97.40 & 98.60 & 98.00 & 98.90 & 95.50 \\[3pt]
\bottomrule
\end{tabular}%
}
\end{table*}

\begin{table}[!htbp]
\centering
\small
\caption{Hyperparameter sweep over $\beta$. AUROC (\%) $\uparrow$ for ViT-16 and ViT-32 backbones.(ID dataset: Clear100, OOD: COCO)}
\label{tab:sweep_beta}
\setlength{\tabcolsep}{6pt}
\begin{tabular}{c cc}
\toprule
$\beta$ & \makecell{ViT-16\\AUROC (\%)$\uparrow$} & \makecell{ViT-32\\AUROC (\%)$\uparrow$} \\
\midrule
0.0 & 91.20 & 89.40 \\
0.5 & 93.80 & 91.90 \\
1.0 & \textbf{96.32} & 93.60 \\
1.5 & 95.70 & 94.80 \\
2.0 & 94.80 & 94.10 \\
3.0 & 93.10 & 92.40 \\
4.0 & 91.40 & 90.80 \\
5.0 & 89.20 & 88.60 \\
6.0 & 86.80 & 86.20 \\
8.0 & 83.40 & 82.70 \\
\bottomrule
\end{tabular}
\end{table}

\begin{table}[!htbp]
\centering
\small
\caption{Hyperparameter sweep over $\eta$. AUROC (\%) $\uparrow$ for ViT-16 and ViT-32 backbones.(ID dataset: Clear100, OOD: COCO)}
\label{tab:sweep_eta}
\setlength{\tabcolsep}{6pt}
\begin{tabular}{c cc}
\toprule
$\eta$ & \makecell{ViT-16\\AUROC (\%)$\uparrow$} & \makecell{ViT-32\\AUROC (\%)$\uparrow$} \\
\midrule
0.0 & 94.10 & 92.40 \\
0.5 & \textbf{96.32} & 94.60 \\
1.0 & 95.20 & 93.50 \\
1.5 & 93.80 & 92.10 \\
2.0 & 91.90 & 90.20 \\
3.0 & 89.40 & 87.80 \\
5.0 & 85.60 & 84.10 \\
\bottomrule
\end{tabular}
\end{table}

\begin{table}[!htbp]
\centering
\small
\caption{Hyperparameter sweep over $\gamma_{\mathrm{cap}}$. AUROC (\%) $\uparrow$ for ViT-16 and ViT-32 backbones.(ID dataset: Clear100, OOD: COCO)}
\label{tab:sweep_gamma_cap}
\setlength{\tabcolsep}{6pt}
\begin{tabular}{c cc}
\toprule
$\gamma_{\mathrm{cap}}$ & \makecell{ViT-16\\AUROC (\%)$\uparrow$} & \makecell{ViT-32\\AUROC (\%)$\uparrow$} \\
\midrule
0.00 & 86.40 & 84.20 \\
0.02 & 88.80 & 86.60 \\
0.05 & 91.50 & 89.40 \\
0.07 & 93.90 & 91.80 \\
0.10 & \textbf{96.32} & 94.10 \\
0.15 & 95.80 & 95.60 \\
0.20 & 94.60 & 94.30 \\
0.30 & 92.40 & 92.10 \\
0.50 & 89.80 & 89.50 \\
\bottomrule
\end{tabular}
\end{table}

\section{Implementation Details}

\paragraph{Backbone and Encoders.}
T-QPM builds on a frozen  CLIP backbone with either a ViT-B/16 or ViT-B/32
visual encoder ($d=512$). Both the visual encoder $\phi^V$ and text encoder $\phi^T$
are kept entirely frozen throughout all phases of training; no fine-tuning of backbone
parameters is performed. Only two scalar fusion parameters, $\tilde{\beta}$ and
$\tilde{\eta}$, are optimized via gradient descent, with effective weights obtained
as $\beta = \log(1+e^{\tilde{\beta}})$ and $\eta = \log(1+e^{\tilde{\eta}})$
(softplus) to enforce strict positivity. These are initialized at
$\tilde{\beta}=1.0$ and $\tilde{\eta}=0.5$, corresponding to
$\beta \approx 1.31$ and $\eta \approx 0.97$ at the start of training.

\paragraph{ID Text Bank Construction.}
The ID text bank $\mathbf{T}^{\mathrm{ID}} \in \mathbb{R}^{K \times d}$ is
constructed once at initialization via prompt ensembling and remains fixed across
all timesteps. For each class $k$, we encode all $P$ prompt templates (loaded from
\texttt{prompt.txt}) through $\phi^T$, $\ell_2$-normalize each embedding, sum
across templates, and re-normalize. This follows the standard CLIP zero-shot
ensembling protocol.

\paragraph{Visual Prototype Construction.}
At each timestep $t$, per-class visual prototypes $\{\boldsymbol{\mu}_{k,t}\}_{k=1}^{K}$
are recomputed from the current timestep's ID training split. For each image,
we extract the class-attended global feature using the DPM-style spatial
attention mechanism (Phase II), normalize it, and accumulate a per-class sum.
The prototype $\boldsymbol{\mu}_{k,t}$ is the $\ell_2$-normalized mean of all
class-$k$ features. Prototypes are computed exclusively from ID data and are
never exposed to OOD samples.

\paragraph{Optimization.}
Training is sequential across $\mathcal{T}=10$ timesteps using the Adam
optimizer with learning rate $\alpha_{\mathrm{lr}} = 3\times 10^{-3}$, batch
size $64$, and $E=5$ epoch per timestep. The total loss at each mini-batch is:
\begin{equation}
\mathcal{L}_{\mathrm{TOTAL}} = \mathcal{L}_{\mathrm{ID}}
  + \lambda_{\mathrm{cov}}\,\mathcal{L}_{\mathrm{COV}}
  + \lambda_{\mathrm{temp}}\,\mathcal{L}_{\mathrm{TEMP}},
\end{equation}
where $\lambda_{\mathrm{cov}}=0.5$ and $\lambda_{\mathrm{temp}}=1.0$.
$\mathcal{L}_{\mathrm{ID}}$ is a balanced CE loss averaged over
clean and covariate-shifted views. $\mathcal{L}_{\mathrm{COV}}$ is the mean
absolute difference between fused scores of clean and shifted pairs.
$\mathcal{L}_{\mathrm{TEMP}}$ is the two-sided ATC drift penalty between
consecutive timesteps, computed as:
\begin{equation}
\mathcal{L}_{\mathrm{TEMP}} =
  |\mathrm{ATC}_t^{\mathrm{clean}} - \mathrm{ATC}_{t-1}^{\mathrm{clean}}|
+ |\mathrm{ATC}_t^{\mathrm{shift}} - \mathrm{ATC}_{t-1}^{\mathrm{shift}}|,
\end{equation}
where the soft-ATC is a differentiable relaxation of the above-threshold coverage:
\begin{equation}
\mathrm{ATC}_t = \mathbb{E}\left[\sigma\!\left(\frac{\delta - S_{\mathrm{FUSED}}}{\kappa}\right)\right],
\quad \kappa = 0.1.
\end{equation}
At $t=0$, $\mathcal{L}_{\mathrm{TEMP}}=0$ since no previous ATC exists.

\paragraph{Threshold Calibration.}
The detection threshold $\delta$ is calibrated once at $t=0$ as the
$\delta_q = 0.01$ quantile of $S_{\mathrm{FUSED}}$ evaluated on the $t=0$ ID
training split, and is held fixed for all subsequent timesteps. This conservative
quantile ensures that fewer than $1\%$ of clean ID training samples fall below
the threshold, directly minimizing the false negative rate at the calibration
timestep.

\paragraph{Covariate Corruption Pipeline.}
Shifted views for training are generated on-the-fly. Gaussian blur is applied
with kernel size $9$ and $\sigma$ uniformly sampled from $[0.1,\;2.0]$.
JPEG compression is applied at a randomly sampled quality level. All corruptions
are applied in the dataloader using torchvision transforms, with no storage of
pre-corrupted images. The spatial attention weight is $\gamma = 0.2$ and the
CLIP logit temperature is $T = 1.0$ throughout.
\begin{figure*}[!t]
    \centering
    \begin{subfigure}{0.48\linewidth}
        \includegraphics[width=\linewidth]{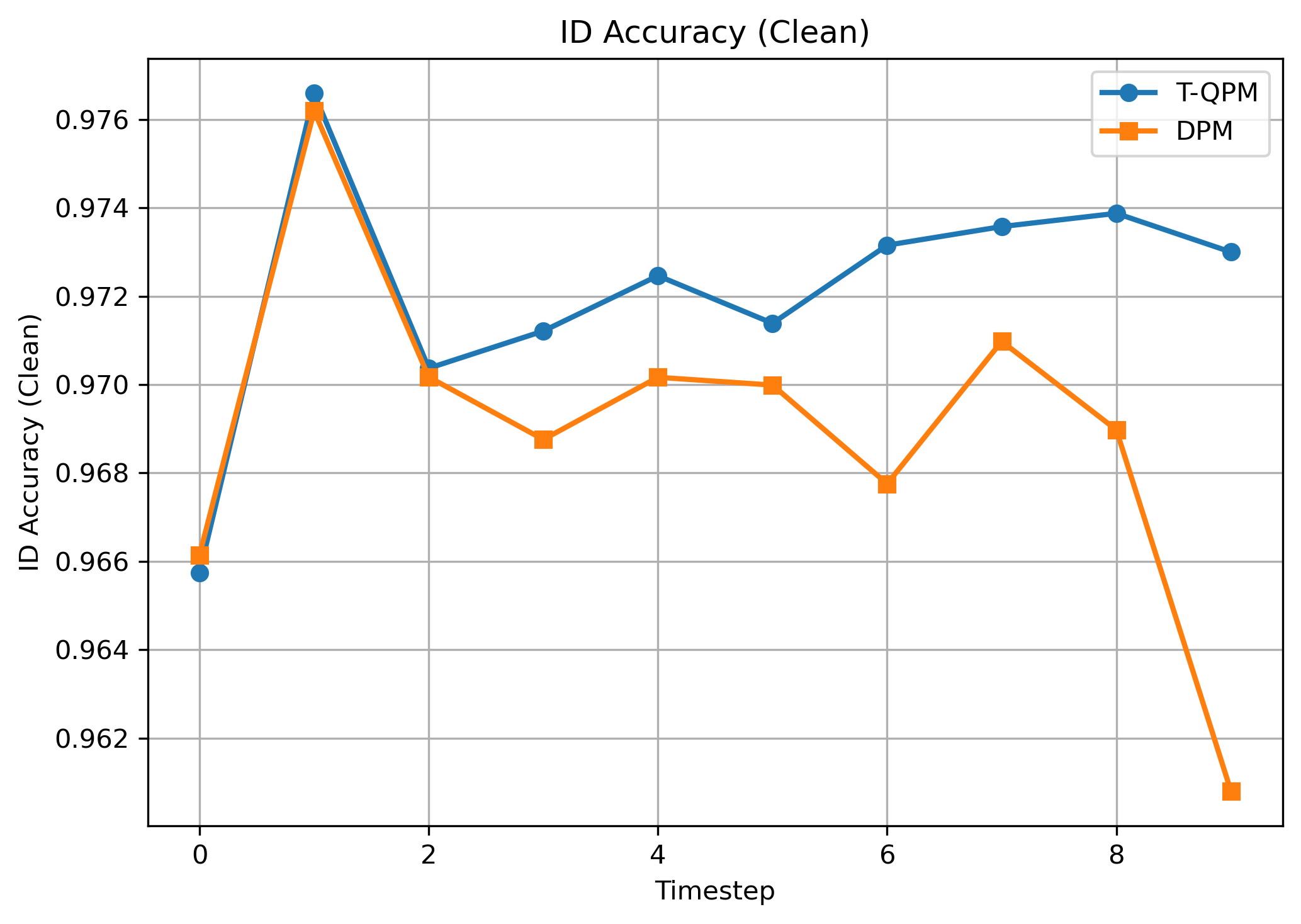}
    \end{subfigure}
    \hfill
    \begin{subfigure}{0.48\linewidth}
        \includegraphics[width=\linewidth]{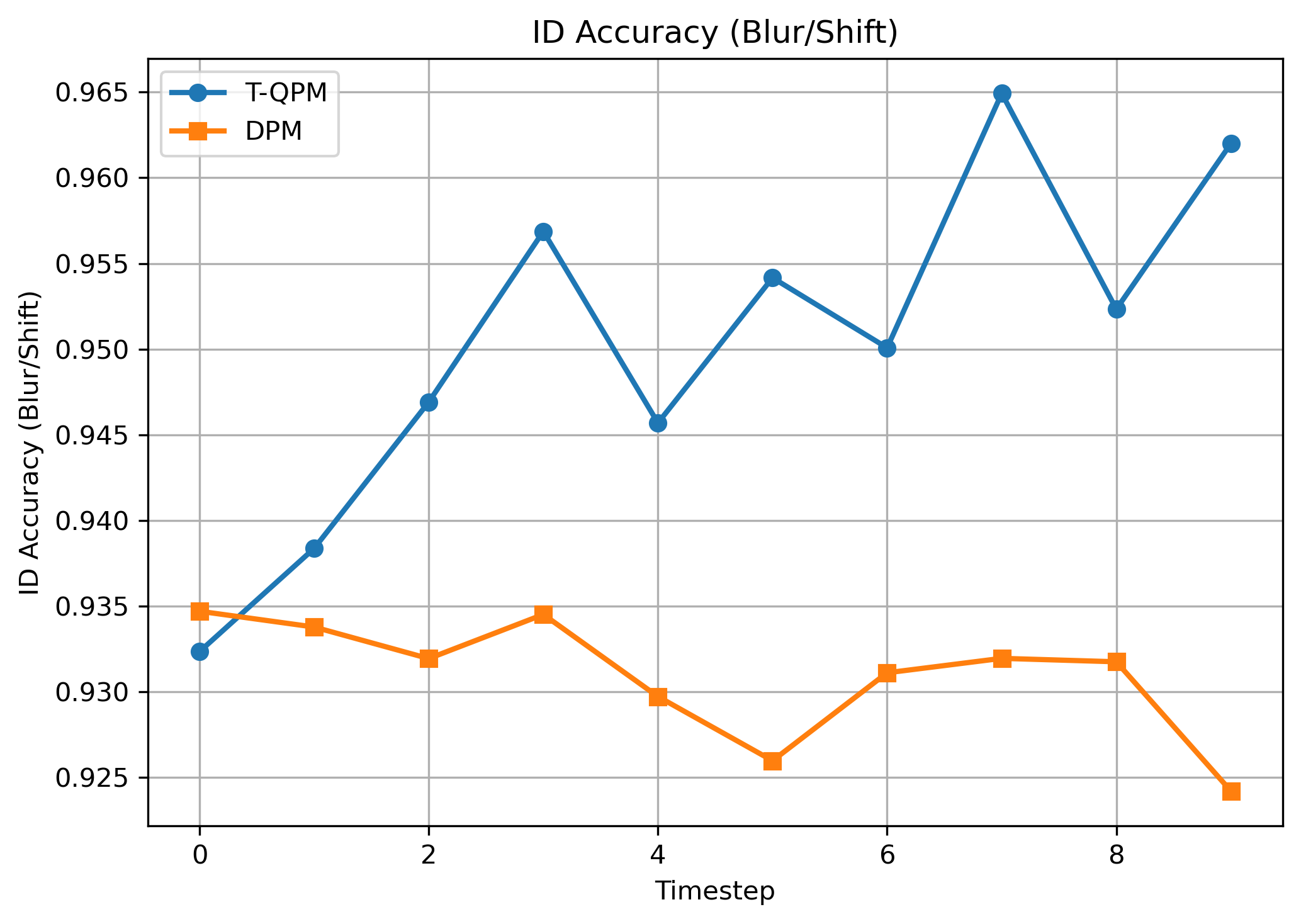}
    \end{subfigure}
    \caption{ID classification accuracy on clean (left) and Gaussian blur-shifted (right) CLEAR100 test sets across all timesteps. Results are averaged over 3 trials.}
    \label{fig:acc_blur}
\end{figure*}

\begin{figure}[H]
    \centering
    \begin{subfigure}{0.7\linewidth}
        \includegraphics[width=\linewidth]{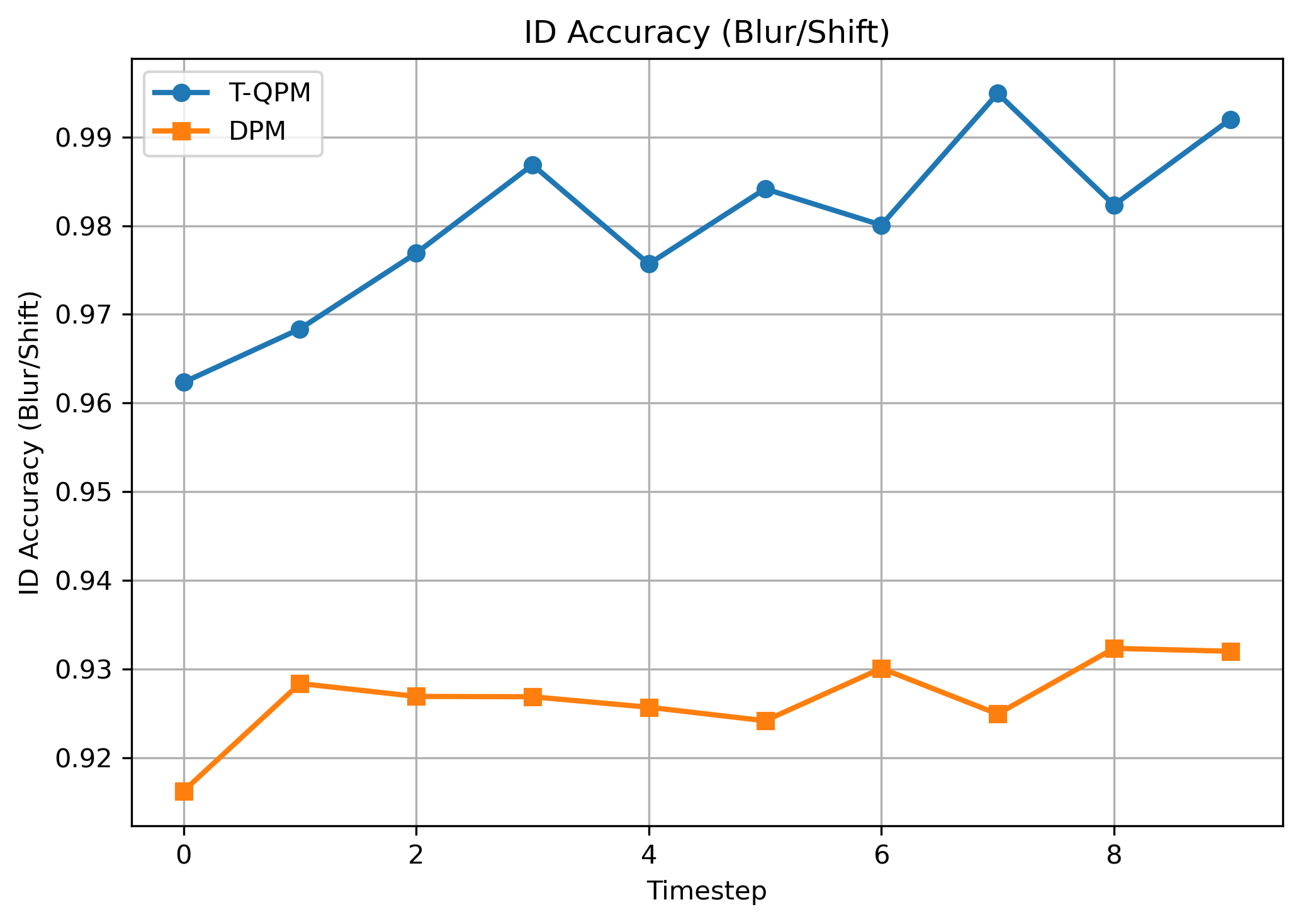}
   \end{subfigure}
      \caption{JPEG-compressed CLEAR100 test sets across all timesteps. Results are averaged over 3 trials.}
    \label{fig:acc_jpeg}
\end{figure}

Under Gaussian blur corruption (Figure~\ref{fig:acc_blur}), T-QPM consistently 
outperforms DPM across all timesteps on both clean and shifted variants. On clean data, 
both methods begin at comparable accuracy ($\sim$0.966), but T-QPM maintains a stable 
upward trend, reaching $\sim$0.974 by $t{=}8$, while DPM exhibits high variance and 
collapses sharply to $\sim$0.961 at the final timestep. The performance gap is substantially 
amplified under covariate shift: T-QPM sustains blur-shifted accuracy in the range 
0.945--0.965, whereas DPM fluctuates between 0.925--0.935 throughout, indicating that 
T-QPM's quadruple matching better preserves discriminative features under 
low-frequency visual degradation.
Under JPEG compression corruption (Figure~\ref{fig:acc_jpeg}), T-QPM demonstrates an even more 
pronounced advantage. On clean data, T-QPM improves steadily from 0.971 at $t{=}0$ to 
$\sim$0.979 by $t{=}8$, while DPM again degrades sharply at the final timestep ($\sim$0.959). 
More strikingly, on JPEG-corrupted inputs, T-QPM exhibits a consistent upward trajectory 
across all timesteps, reaching $\sim$0.991 at $t{=}7$---while DPM remains nearly flat in 
the 0.920--0.932 range throughout. This $\sim$5--6\% sustained gap under JPEG shift 
suggests that T-QPM's interference-based scoring mechanism is particularly robust to 
high-frequency compression artifacts, which tend to destabilize standard softmax-based 
confidence estimates. Taken together, both figures demonstrate that T-QPM not only 
maintains higher clean accuracy but generalizes significantly better under realistic 
covariate corruptions as temporal drift accumulates.

\paragraph{OOD Dataset Streaming.}
All OOD datasets are used exclusively at inference, never during training.
COCO~\cite{lin2015microsoftcococommonobjects} is loaded from local disk along with captions.
Flickr30K~\cite{plummer2016flickr30kentitiescollectingregiontophrase},
ImageNet-1K-VL-Enriched~\cite{imagenet_captions_hf}, and
CC12M~\cite{changpinyo2021cc12m} are streamed via the HuggingFace
\texttt{datasets} library with a reservoir shuffle buffer of $10{,}000$,
capped at $20{,}000$, $10{,}000$, and $10{,}000$ examples per evaluation,
respectively. Captions for Flickr30K are selected uniformly at random from
the available per-image candidates. A new streaming iterator is instantiated
at each timestep to avoid exhausting the stream.

\paragraph{Reproducibility.}
All experiments use a fixed random seed (default: $1556$), set across
random, numpy, torch, and torch.cuda.
Results are averaged over 3 independent trials with seeds offset by
trial\_id $\in \{0,1,2\}$. All experiments are run on a single
NVIDIA GPU with num\_workers$=4$ for ID dataloaders and
num\_workers$=0$ for HuggingFace streaming OOD loaders.

%% file: sec/2_extensive-related_work.tex
\section{Extended Related Work}

\subsection{Vision-Language Models for OOD Detection}

The emergence of large-scale vision-language models (VLMs), particularly CLIP~\cite{radford2021clip}, has enabled new approaches to out-of-distribution (OOD) detection that leverage semantic information from both visual and textual modalities. MCM~\cite{ming2022mcm} introduces a training-free OOD detection framework that treats textual class embeddings as semantic prototypes and measures the alignment between visual representations and class concepts in the shared CLIP embedding space. MCM demonstrates that multimodal representations significantly improve OOD detection performance compared to conventional image-only approaches. Similarly, ZOC~\cite{esmaeilpour2022zero} utilizes CLIP representations to derive textual descriptions for unseen classes and perform zero-shot OOD detection.

More recently, LoCoOp~\cite{miyai2023locoopfewshotoutofdistributiondetection} improves OOD detection through prompt learning, while NegLabel~\cite{jiang2024negativelabelguidedood} incorporates negative label guidance to better separate in-distribution (ID) and OOD samples in the CLIP embedding space. Among existing VLM-based methods, DPM~\cite{zhang2023dpm} is most closely related to our work. DPM extends CLIP-based OOD detection through dual-pattern matching, combining semantic text matching with visual pattern matching derived from ID training statistics. DPM further introduces DPM-T, which incorporates learnable prompts and projection layers for task adaptation. However, existing VLM OOD detectors assume static deployment environments and do not explicitly model temporal distribution shift. In contrast, our proposed T-QPM extends pattern matching to a temporal multimodal setting by leveraging image-caption pairs, temporal visual prototypes, and four complementary cross-modal matching signals.

\subsection{OWL and OOD Generalization}

Open-world learning (OWL) requires models to simultaneously generalize under distribution shifts and detect semantic OOD samples encountered during deployment~\cite{zhu2024owl}. Existing OOD detection methods typically focus either on semantic OOD detection or domain generalization, while treating these objectives independently.

SCONE~\cite{bai2025feedbirdssconeexploiting} addresses this challenge by introducing an energy-margin framework that jointly separates ID samples, covariate-shifted samples, and semantic OOD samples using unlabeled wild data. Temp-SCONE~\cite{naiknaware2025tempscone} extends SCONE to dynamic environments through temporal regularization based on Average Thresholded Confidence (ATC), encouraging confidence stability across consecutive timesteps.

While Temp-SCONE demonstrates the importance of temporal consistency under evolving distributions, it operates entirely in a unimodal setting using visual features and energy-based scoring. In contrast, T-QPM is designed specifically for multimodal vision-language models and performs OOD detection using image-caption pairs. Furthermore, ATC serves a different purpose in the two frameworks. In Temp-SCONE, ATC is directly used as a confidence signal for temporal regularization, whereas in T-QPM ATC is employed solely as a temporal boundary stabilization mechanism, while the final OOD decision is determined by a fused quadruple-pattern matching score derived from multimodal image-text interactions.

Complementary approaches such as TENT~\cite{wang2020tent} and continual adaptation frameworks~\cite{wu2023meta} attempt to stabilize predictions under distribution shifts through online adaptation. However, these methods do not explicitly distinguish between covariate shift and semantic OOD detection and are not designed for multimodal temporal OOD settings.

\subsection{Temporal Domain Generalization}

Temporal distribution shift has recently emerged as an important challenge in machine learning due to the continuously evolving nature of real-world data. Benchmarks such as CLEAR~\cite{lin2022clearbenchmarkcontinuallearning}, Yearbook, and other temporally partitioned datasets have motivated the development of temporal domain generalization methods that seek to maintain performance as data distributions evolve over time.

Existing temporal generalization approaches primarily focus on classification or regression tasks and evaluate performance using predictive accuracy or regression error. Continuous temporal domain generalization methods typically address adaptation across evolving domains but do not consider semantic OOD detection. Consequently, they cannot directly identify previously unseen classes or detect novel semantic concepts during deployment.

Our work differs from temporal domain generalization approaches in two important ways. First, we address the more challenging open-world setting where both temporal distribution shift and semantic OOD detection must be handled simultaneously. Second, we leverage multimodal image-caption information through a temporal quadruple-pattern matching framework, enabling robust OOD detection and covariate-shift generalization under continuously evolving distributions.